\documentclass[twoside,11pt]{article}
\usepackage{jair}

\usepackage{theapa}
\usepackage{rawfonts}
\usepackage{adjustbox}
\usepackage{amsfonts}
\usepackage{amssymb}
\usepackage{amsthm}
\usepackage{amsmath}
\newtheorem{theorem}{Theorem}
\usepackage{graphicx,psfrag,epsf}
\usepackage{enumerate}
\usepackage{tabularx}
\usepackage{url} 

\usepackage{floatrow}
\DeclareFloatFont{small}{\small}
\floatstyle{plaintop}
\restylefloat{table}
\usepackage{soul}
\usepackage[normalem]{ulem}
\newcommand{\stkout}[1]{\ifmmode\text{\sout{\ensuremath{#1}}}\else\sout{#1}\fi}
\usepackage{comment}

\usepackage{adjustbox}
\usepackage[makeroom]{cancel}
\usepackage{bbold}
\usepackage{float}
\usepackage{bm}

\usepackage[implicit=false]{hyperref}

\hypersetup{colorlinks=true,
            linkcolor=red,
            urlcolor=blue,
            citecolor=blue}

\usepackage{algorithm}
\usepackage{algpseudocode}
\usepackage{color}

\usepackage{etoolbox}
\usepackage{csvsimple}

\newcommand{\M}{\mathfrak{m}}
\newcommand{\MS}{\mathcal M}

\newtheorem{constraint}{Constraint}
\definecolor{Florian}{RGB}{255,0,0}

\definecolor{Geir}{RGB}{0,0,255}

\definecolor{AH}{RGB}{255,0,255}

\definecolor{todo}{RGB}{10,200,10}

\usepackage{siunitx}

\jairheading{72}{2021}{901-942}{05/21}{11/21}
\ShortHeadings{Flexible Bayesian Nonlinear Model Configuration}
{Hubin, Storvik \& Frommlet}
\firstpageno{1}

\begin{document}

\title{Flexible Bayesian Nonlinear Model Configuration}

\author{\name Aliaksandr Hubin \email aliaksah@math.uio.no \\
        \addr Department of Mathematics, University of Oslo\\
       Oslo, 0316, Norway\\
       Norwegian Computing Center\\ 
       Oslo, N-0314, Norway\\
        \AND
       \name Geir Storvik \email geirs@math.uio.no \\
       \addr Department of Mathematics, University of Oslo\\
       Oslo, 0316, Norway\\
       \AND
       \name Florian Frommlet \email florian.frommlet@meduniwien.ac.at \\
       \addr CEMSIIS, Medical University of Vienna\\
       Vienna, 1090, Austria
       }


\maketitle
\setcounter{page}{901}

\begin{abstract}
Regression models are used in a wide range of applications providing a powerful scientific tool for researchers from different fields. 
Linear, or simple parametric, models are often not sufficient to describe complex relationships between input variables and a response. Such relationships can be better described through flexible approaches such as neural networks, but this results in less interpretable models and potential overfitting. Alternatively, specific parametric nonlinear functions can be used, but the specification of such functions is in general complicated.
In this paper, we introduce a flexible approach for the construction and selection of highly flexible nonlinear parametric regression models.
Nonlinear features are generated hierarchically, similarly to deep learning, but have additional flexibility on the possible
types of features to be considered. This flexibility, combined with
variable selection, allows us to find a small set of \emph{important} features and thereby
more interpretable models.  Within the space of possible functions, a Bayesian approach, introducing priors for functions based on their complexity, is considered. 
A genetically modified mode jumping Markov chain Monte Carlo algorithm is adopted to perform Bayesian inference and estimate posterior probabilities for model averaging. 
In various applications, we illustrate how our approach is used to obtain meaningful nonlinear models. Additionally, we compare its predictive performance with several machine learning algorithms.
\end{abstract}

\section{Introduction}\label{section1}

Regression and classification methods are indispensable tools for answering scientific questions in almost all research areas. 
Modern technologies have led to the paradigm of machine learning, where large sets of explanatory variables are routinely considered. Deep learning procedures have become quite popular and highly successful in a variety of real-world applications \cite{Goodfellow-et-al-2016}. 
Such procedures often outperform traditional methods, even when the latter are carefully designed and reflect expert knowledge~\cite{refenes1994stock,razi2005comparative,adya1998ective,sargent2001comparison,kanter2015deep}. The main reason for this is that the features from the outer layer of the deep neural networks become highly predictive after being processed through numerous nonlinear transformations. Specific regularization techniques (dropout, $L_1$ and $L_2$ penalties on the weights,  etc.) have been developed for deep learning procedures to avoid overfitting to training data.
Nonetheless, one usually has to use huge data sets to be able to produce generalizable neural networks.

Deep learning models are typically heavily over-parametrized and difficult to interpret. These models are densely approximating the function of interest and transparency is traditionally not a goal in their applications. Depending on the context, this can be a more or less severe limitation. In many research areas, it is desirable to obtain interpretable (nonlinear) regression models rather than just some dense representation of them. Moreover, in situations where the law requires explainability, one needs a legally compliant and transparent method \cite{molnar2019,aas2019explaining}. This occurs, for example, in the ‘right to explanation’ of the European Union's General Data Protection Regulation.

 In this paper, we present an approach to construct regression models 
 which retain much of the flexibility of modern machine learning methods but achieve much simpler structures through prior constraints, both hard and soft, within a Bayesian framework. A new class of models is introduced which we call \emph{Bayesian generalized nonlinear regression models} (BGNLMs). This class is based on generalized linear models~\cite{nelder1972generalized} where the distribution of the observations is assumed to be from the exponential family, but with the mean parameter being a nonlinear function of the input variables. The specific modeling of this nonlinear dependency will include ideas related to neural networks. We will also briefly discuss how BGNLM can be extended to generalized linear mixed models, which include random effects to model overdispersion or specific dependence structures of observations.  
 
 A variety of classes of nonlinear models exist in the literature, see e.g~\citeA{turner2007generalized} and the references therein.
 Our approach extends these classes through a hierarchically defined construction.  The resulting model topology resembles highway networks~\cite{srivastava2015highway} and densely connected convolutional networks~\cite{huang2017densely}
 in that new features are constructed based on features of several previous layers. Our approach also allows for automatic feature selection and thereby automatic
architecture specification, which can sparsify the model drastically. In this sense, it resembles symbolic regression combined with genetic programming~\cite{koza1994genetic}. 
 As it typically holds for Bayesian approaches in general, we tend to obtain robust solutions with relatively little overfitting.

Fitting BGNLMs is based on an efficient search algorithm that combines ideas from MCMC, mode jumping, and genetic programming. A similar algorithm was previously introduced and justified in the context of Bayesian logic regression~\cite{hubin2018novel}. 
We demonstrate that automatic feature engineering within regression models combined with Bayesian variable selection and model averaging can retain the predictive abilities of advanced statistical models while keeping them reasonably simple, interpretable, and transparent.  

Like many other Bayesian approaches, our approach is computationally intensive. Although we provide a fully Bayesian framework, in practice we recommend a partially empirical Bayes approach where some of the parameters involved are estimated more directly from the data. Compared with the fully Bayesian approach, this results in a considerable increase in computational speed and more interpretable models while not losing too much in predictive performance.
Our current implementation can handle up to a few hundred input variables and several thousand observations. We hope that the development of more sophisticated numerical algorithms will make the approach applicable to even larger data sets in the future. 

One of our major goals is to illustrate the potential of BGNLM to find meaningful nonlinear models in applications. To demonstrate this, we retrieve several ground physical laws in closed form (by which we understand the commonly known equations for the corresponding physical laws) from raw data. We also demonstrate the power of our approach to retrieve complex logical expressions from a data-generative model. Furthermore, the predictive ability of BGNLM is compared with deep neural networks, random forests, boosting procedures, and several other machine learning techniques under various scenarios. Last but not least, we provide an experiment on the extension of BGNLM which allows selecting latent Gaussian variables and thus model possible over-dispersion and spatio-temporal structure between the observations.

The rest of the paper is organized as follows: The class of BGNLMs is mathematically defined in Section~\ref{section2}. In Section~\ref{section3}, we describe the algorithm used for inference. Section~\ref{extensions} considers the extension of BGNLM to mixed model settings and a fully Bayesian extension of BGNLM. In Section~\ref{section4}, the BGNLM is applied to several real data sets:        
The first examples are concerned with classification and regression tasks where the performance of our approach is compared with various competing statistical learning algorithms. 
Later examples have the specific goal of retrieving an interpretable model. Section~\ref{section5} concludes and provides some suggestions for further research. Additional examples and details are provided in the Appendix.

\section{{Bayesian Generalized Nonlinear Models}}\label{section2}

We want to model the relationship between $m$ explanatory variables and a response variable based on $n$ samples from a training data set. For $i=1,...,n$, let $Y_{i}$, denote the response data and $\bm{x}_{i} = (x_{i1}, \dots, x_{im})$  the corresponding $m$-dimensional vector of input covariates. The proposed model 
extends the framework of generalized linear models \cite[GLM]{McCullagh-Nelder-1989} to include a flexible class of nonlinear transformations of covariates. The members of this class,  $F_{j}(\bm{x},\bm\alpha_j)$, $j = 1,...,q$, will be called features, where $q$ is finite but potentially huge. Details of the feature generating process are provided in Section~\ref{sec:feature}. 
The BGNLM is then defined as follows: 
\begin{subequations}\label{themodeleq}
\begin{align} 
  Y|\mu,\phi \sim & \ \mathfrak{f}(y|\mu,\phi);\\
   \mathsf{h}(\mu) = &\ \beta_0  + \sum_{j=1}^{q} \gamma_{j}\beta_{j}F_{j}(\bm{x},\bm\alpha_j)\; .
   \label{DeepModel}
\end{align}
\end{subequations}
Here, $\mathfrak{f}(\cdot|\mu,\phi)$ is the density (mass) of a probability distribution from the exponential family with mean $\mu$ and dispersion parameter $\phi$, while $\mathsf{h}(\cdot)$ is a link function relating the mean to the underlying features.  Feature $F_j$ depends on a (potentially empty) set of internal parameters $\bm\alpha_j$ which is described below. The features enter the model with coefficients $\beta_j \in \mathbb{R}, j = 1,...,q$.  Our formulation in~\eqref{DeepModel} enumerates all possible $q$ features but uses binary variables $\gamma_j\in \{0,1\}$ to indicate whether the corresponding features are to be included in the model or not.
Priors for the different parameters of the model are specified in Section \ref{sec:priors} and are designed to favor sparse models.

\subsection{The Feature Generating Processes}\label{sec:feature}
The basic building block of neural networks is the neuron, which consists of a nonlinear transformation applied to a linear combination of input variables. 
In multi-layer neural networks, neurons are arranged in multiple layers leading to an iterative application of these specific nonlinear transformations. Our \emph{feature generating process} is based on the same type of transformation, where the construction of possible features is performed by recursively using nonlinear combinations of any previously defined features. In principle, these features are feature functions, but we will use the term features for a simpler notation. 
The generation of a new feature corresponds to a  mapping from a set of functions into a function space. 

New features are generated using a hierarchical procedure which is based on a mix of ideas from symbolic regression~\cite{koza1994genetic} and neural networks.  The procedure is initialized with the original set of input variables as features, i.e.  $F_j(\bm{x}_{i}) = x_{ij}$ for $j \in \{1,\dots,m\}$.
Assume that, at a certain point, a set of features
$\{F_k(\cdot,\bm\alpha_k),k\in A\}$ is generated, where $A \subset \{1,...,q\}$ are indices of the features that are included in the set. The nature of the parameter sets $\bm\alpha_k$  will be described in detail below. 
Then we define three different transformations to generate a new feature $F_j(\cdot,\bm\alpha_j)$:
\[
F_j(\bm{x},\bm{\alpha}_j)=
\begin{cases}
g_j(\alpha^{out}_{j,0}  + \sum\limits_{k \in A_j\subseteq A }\alpha^{out}_{j,k}F_{k}(\bm{x},\bm\alpha_{k}))&\text{a nonlinear \textit{projection}};\\
g_j(F_k(\bm{x},\bm{\alpha}_{k}))&\text{a \textit{modification}, $k\in A$};\\
F_k(\bm{x},\bm{\alpha}_{k})F_l(\bm{x},\bm{\alpha}_{l})&\text{a \textit{multiplication}, $k,l\in A$.}
\end{cases}
\]
The transformations as defined above are assumed to follow a nested structure in which previously defined features $\{F_k(\bm x,\bm\alpha_k)\}$ keep their $\bm\alpha_k$ parameters fixed. 
Section \ref{sub:Estimate_alpha}  describes two strategies for the specification of the parameters in this way, but also a more general version in which the parameters within the nested features are allowed to be updated with respect to new operations. A fully Bayesian framework is also considered in Section~\ref{sec:fullb}.

The first transformation called \textit{projection} relates directly to the neuron. It is similar to the transformations used in neural networks but with the added flexibility that the nonlinear transformation $g_j$ can be selected from a class of functions $\mathcal{G}$. All functions within $\mathcal{G}$ should have domain $\mathbb{R}$ and range within $\mathbb{R}$.  The linear combination is taken over a subset of existing features $\{F_{k}(\cdot,\bm\alpha_k),k\in A_j\}$ where $A_j\subseteq A$ and the cardinality of $|A_j|$ is larger than 1.  The set of parameters  $\bm\alpha_j$ is defined by considering two subsets  $\bm\alpha_j = \{\bm\alpha^{out}_j,\bm\alpha^{in}_j\}$. The first one, $\bm\alpha^{out}_j$, denotes the $|A_j|$ parameters plus the intercept $\alpha^{out}_{j,0}$ which describe the current projection, whereas $\bm\alpha^{in}_j = \cup_{k \in A_j} \bm\alpha_k$ collects all the parameters involved in the nested features of the projection. 

To facilitate the generation of parsimonious nonlinear features, we introduce two additional transformations: \textit{modification} and  \textit{multiplication}.
Note that for both modifications and multiplications, it holds that $\bm\alpha^{out}_j = \emptyset$. Modification allows for nonlinear transformations of an existing feature.
Multiplication of two different features corresponds to interactions in the language of statistics. The multiplication transformation is allowed to select the same feature twice, i.e. $l = k$ is admissible. In principle, both of these transformations can be seen as special cases of projection. Modification is a special case of projection for which $|A_j|=1$. To increase interpretability of the features, both multiplication (which can be seen as a special case of two projection operations with the $\exp(x)$ and $\log(x)$ transformations) and modification are defined as separate transformations.  This also allows limiting the BGNLMs to only include modifications and multiplications as discussed in Section~\ref{sec:gmjmcmc}.

\paragraph{Feature properties} Define the \emph{depth}, $d_j$, of a projection and modification based feature with index $j$ as the minimum number of nonlinear transformations applied recursively when generating that feature. For example, a feature $F_j(\bm x,\bm{\alpha}_j) = \sin\left(\cos\left(\tan(x_1)\right)+\exp(x_2)\right)$ has depth $d_j = 3$. If a multiplication is applied, the depth is defined as one plus the sum of depths of its operands. For example, $F_k(x,\bm{\alpha}_k) = x_2\exp(x_1)$ has depth $d_k = 2$ (where we have used that the depth of a linear component is zero). Define the \emph{local width}, $lw_j$, to be the number of previously defined features that are used to generate a new feature. This is $|A_j|$ for a projection, 1 for a modification, and 2 for a multiplication (including the case of $k = l$, see $x^2, x^3$ in Table~\ref{tab:list_features} below).  
Finally, the \textit{operations count}, $oc_j$, of a feature  is the total number of algebraic operations (additions, multiplications between the features, and transformations) used in its representation.  For example, $F_j(\bm x,\bm{\alpha}_j)=x$ has $oc_j=0$, $F_j(\bm x,\bm{\alpha}_j)=v(x)$ has $oc_j=1$, $F_j(\bm x,\bm{\alpha}_j)=v(u(x))$ has $oc_j=2$, and $F_j(\bm x,\bm{\alpha}_j)=v(u(x)+x)$ has $oc_j=3$. Further examples indicating the structure of the \emph{depth}, \emph{local width}, and \emph{oc} measures are given in Table~\ref{tab:list_features}.

\begin{table}[!ht]
    \tabcolsep=0.11cm
    \begin{tabular}{lccl|lccl}
    \hline
    Feature&$d$&$oc$&$lw$&Feature&$d$&$oc$&$lw$\\
    \hline
    $x$   &0&0&1\\
    \hline
    $u(x)$&1&1&1         &$x^2$&1&1&2\\
    $v(x)$&1&1&1&&\\
    \hline
    $u(u(x))$&2&2&1    &$xu(x)$&2&2&2\\  
    $v(u(x))$&2&2&1     &$xv(x)$&2&2&2\\ 
    $u(v(x))$&2&2&1      &$x^3$&2&2&2\\
    $v(v(x))$&2&2&1      &$u(x^2)$&2&2&1\\
    $u(x+u(x))$&2&3&2  &$v(x^2)$&2&2&1\\
    $v(x+u(x))$&2&3&2  &$u(x+x^2)$&2&3&2\\
    $u(x+v(x))$&2&3&2 &$v(x+x^2)$&2&3&2\\
    $v(x+v(x))$&2&3&2  &$u(u(x)+x^2)$&2&4&2\\
    $u(u(x)+v(x))$&2&4&2&$v(u(x)+x^2)$&2&4&2\\
    $v(u(x)+v(x))$&2&4&2&$u(v(x)+x^2)$&2&4&2\\
    $u(x+u(x)+v(x))$&2&5&3&$v(v(x)+x^2)$&2&4&2\\
    $v(x+u(x)+v(x))$&2&5&3&$u(x + u(x)+x^2)$&2&3&3\\
    &&&&$v(x + u(x)+x^2)$&2&5&3\\
    &&&&$u(x + v(x)+x^2)$&2&5&3\\
    &&&&$v(x + v(x)+x^2)$&2&5&3\\
    &&&&$u(u(x) + v(x)+x^2)$&2&6&3\\
    &&&&$v(u(x) + v(x)+x^2)$&2&6&3\\
    &&&&$u(x+u(x) + v(x)+x^2)$&2&7&4\\
    &&&&$v(x+u(x) + v(x)+x^2)$&2&7&4\\
    \hline 
    \end{tabular}
       \caption{Feature space with $m = 1$, $\mathcal{G} = \{u(x),v(x)\}$ and depth $d\leq 2$. Here $oc$ is the operations count, and $lw$ is the local width. The left panel includes the original $x$ and features generated by projection or modification while the right panel involve features generated by an additional multiplication transformation. Also, here $x^2$ is generated by a multiplication as $x\times x$, and $x^3$ as $x\times x\times x$.}
    \label{tab:list_features}
\end{table}

As an illustration, consider the case with $m=1$ and  $\mathcal{G}=\{u(\cdot),v(\cdot)\}$. Table~\ref{tab:list_features} lists all possible features of depth $d \leq 2$ together with their local widths and operations counts,  where the non-zero $\bm\alpha_j$'s have been set equal to one and the intercept to zero. Already this simple example indicates that
the number of features grows\textit{ super-exponentially} with depth. To see this more formally, observe that the number of features of depth $d$ obtained only from projections and modifications is of the form $q^l_d=|\mathcal{G}|\left(2^{{\sum_{t=0}^{d-1}q^l_t}}- 1\right) -\sum_{t=1}^{d-1}q^l_t $, where $q^l_0=m$ and $|\mathcal{G}|$ denotes the number of different functions included in $\mathcal{G}$. Hence, $q^l_d$ gives a lower bound for the total number of features of the corresponding depth which grows super-exponentially with $d$. The multiplication transformation further adds a considerable number of features of depth $d$ and the combinatorics gets a bit more complex. Let $q_d = q_d^p + q_d^*$, where $q_d^p$ is the number of features of depth $d$ resulting from projections and modifications,  whilst $q_d^*$ is the number of features resulting from multiplications. With $q_0 = m$ and $q^p_0 =  q^*_0 = 0$, the following recursive relationships hold:
\begin{align*}
q_d^p =&  |\mathcal{G}|\left(2^{{\sum_{t=0}^{d-1}q_t}}- 1\right) -\sum_{t=1}^{d-2}q_t - q_{d-1}^p, \\
q_d^* =&  \left\{  \begin{array}{ll}
 \sum\limits_{t < s}  q_{t}q_{d-t-1} + \binom{1 + q_{s}}{2},   & d \mbox{ is odd, $s=(d-1)/2$;}\\
   \sum\limits_{t < s}  q_{t} q_{d-t-1},  & d \mbox{ is even, $s=(d-2)/2$.}
\end{array}
\right.  
\end{align*}
For our toy example, we thus obtain for projections and modifications: $q_0^p = 0, q_1^p = 2, q_2^p = 28$ and $q_3^p = 68\,719\,476\,703$; for multiplications:  $q_0^* = 0, q_1^* = 1, q_2^* = 3$ and $q_3^* = 37$. This gives in total $q_0 = 1, q_1 = 3, q_2 = 31$ and $q_3 = 68\,719\,476\,740$. Note that without multiplications the number of features (corresponding to the lower bound) would have been: $q_0^l = 1, q_1^l = 2, q_2^l = 12$ and $q_3^p = 8176$. For $d \leq 2$, these numbers exactly correspond to the number of features listed in Table~\ref{tab:list_features}. The super-exponential growth means in practise that only a very limited $d$ can be considered even for problems with few covariates. 

\paragraph{Connections to other models} The feature space constructed by the suggested feature generating process is extremely rich and encompasses features from numerous other popular statistical and machine learning approaches as special cases. If the set of nonlinear functions only consists of one specific function, for example, $\mathcal{G} = \{\sigma(x)\}$
where $\sigma(\cdot)$ is the sigmoid function, then the corresponding feature space includes numerous possible densely connected neural networks with the sigmoid activation function. 
Another important class of features included in the BGNLM framework consists of decision trees \cite{breiman1984classification}. Simple decision rules correspond to repeated use of the nonlinear function $g(x)=\text{I}(x\ge 1)$. 
Intervals and higher dimensional regions can be defined through multiplications of such terms. 
Multivariate adaptive regression splines \cite{friedman1991multivariate} are included by allowing a pair of piece-wise linear functions $g(x) = \max\{0,x-t\}$ and $g(x) = \max\{0,t-x\}$. Fractional polynomials~\cite{royston1997approximating} can also be easily included by adding appropriate power functions $g(x) = x^s$ to $\mathcal{G}$. 
Logic regression, characterized by features being logic combinations of binary covariates~\cite{ruczinski2003logic,hubin2018novel}, 
is also fully covered by BGNLM models. 
The BGNLM framework extends these alternatives and allows combinations of different types of features by
defining more than one function in $\mathcal{G}$, resulting in, for example, features like  $\left(0.5x_1+10x_2^{0.5}+3\text{I}(0.2x_2>1)+0.1\sigma(2.5x_3)\right)^2$. 

\subsubsection{Specification of \texorpdfstring{$\alpha$}{a} Parameters} \label{sub:Estimate_alpha}
For the general \textit{projection} 
transformation, one has to specify $\bm{\alpha}_j$ when generating a new feature.  In Section~\ref{sec:fullb}, we describe a fully Bayesian approach that introduces priors for the $\bm{\alpha}_j$ parameters. Unfortunately, the fully Bayesian option is currently not computationally feasible except for problems where the number of input variables is very small. Hence, we describe three different strategies of increasing sophistication to specify $\bm{\alpha}_j$. In the first two strategies, only $\bm{\alpha}_j^{out}$ are estimated whereas $\bm{\alpha}_j^{in}$ are kept fixed which is a restriction inspired by Cascade Learning \cite{fahlman1990cascade}. The third strategy estimates all $\bm{\alpha}_j$ parameters jointly when generating a new feature like in deep learning. All of these strategies are based on finding parameter values that give high explanatory power to $F_j(\bm x,\bm\alpha_j)$ regardless of other features involved in the model.

\paragraph{Strategy 1 (optimize, then transform, na{\i}ve)} Our simplest procedure to obtain $\bm{\alpha}_j$ is to fix $\bm{\alpha}^{in}_j$ from the nested features  and then compute maximum likelihood estimates for $\bm{\alpha}^{out}_j$ by applying model~\eqref{themodeleq} directly without considering the nonlinear transformation $g_j(\cdot)$:\begin{equation*} \label{Strategy_1}
\mathsf{h}(\mu) = \alpha^{out}_{j,0}  \; \; \; + \; \sum_{k \in A_j}\alpha^{out}_{j,k}F_{k}(\bm{x},\bm\alpha_k) \; .
\end{equation*}
This choice of parameter estimators for the generated features has several advantages. The nonlinear transformation $g_j(\cdot)$ is not involved when computing $\bm{\alpha}_j^{out}$. Therefore the procedure can easily be applied on many nonlinear transformations $g_j(\cdot)$ simultaneously. Also, functions that are not differentiable, like the characteristic functions of decision trees or the ReLU function, can be used. Furthermore, maximum likelihood estimation for generalized linear models induces a convex optimization problem, and the obtained $\bm{\alpha}_j^{out}$ are unique. However, neglecting the activation function $g_j(\cdot)$ and fixing $\bm{\alpha}^{in}_j$  results in a feature-generating process that might not deliver the best features in terms of prediction.

\paragraph{Strategy 2 (transform, then optimize, concave)} Like in Strategy 1, the  weights $\bm{\alpha}^{out}_j$  are estimated conditionally on $\bm{\alpha}^{in}_j$, but now optimization is performed after applying the transformation $g_j(\cdot)$. In other words, the weights are obtained as  maximum likelihood estimates for the following model:  
\begin{equation} \label{Strategy_2}
\mathsf{h}(\mu) = g_j\left(\alpha^{out}_{j,0}  \; \; \; + \; \sum_{k \in A_j}\alpha^{out}_{j,k}F_{k}(\bm{x},\bm\alpha_k)\right) \; .
\end{equation}
This strategy yields a particularly simple optimization problem if  $\mathsf{h}^{-1}(g_j(\cdot))$ is a concave function, in which case the estimates are uniquely defined. If we want to use gradient-based optimizers, we have to restrict $\mathsf{h}^{-1}(g_j(\cdot))$ to be continuous and differentiable in the regions of interest. Otherwise, gradient-free 
optimization techniques have to be applied. 

\paragraph{Strategy 3 (transform, then optimize across all layers, deep)} Similarly to Strategy~2, parameters are obtained as maximum likelihood estimates using  model~\eqref{Strategy_2}, but now we jointly estimate the outer $\bm{\alpha}^{out}_j$ together with the nested $\bm{\alpha}^{in}_j$ (the $\bm\alpha_k$ parameters entering from the nested features should in this case formally be denoted by $\bm\alpha_{j,k}$). Hence, the optimization is performed with respect to parameters across all layers.  All of the nonlinear functions involved have to be continuous and differentiable in the regions of interest in order to enable the use of gradient-based optimizers. 
A major drawback with this strategy is that it is not possible to utilize previous specifications of the parameters; all parameters need to be recomputed. 
There is also no guarantee of finding a unique global optimum of the likelihood of the feature, even if all the $g_j$-functions are concave. If gradient-free optimizers are used, the problem becomes extremely computationally demanding. Furthermore, different local optima define different features having structurally the same configuration.

\subsection{Bayesian Model Specifications} \label{sec:priors}

The feature generating process described in the previous section defines an extremely large and flexible feature space. In order to avoid overfitting, we will use a Bayesian approach with a prior giving preference to simple structures through both hard and soft regularization. In this subsection, we assume that the $\bm\alpha_j$ parameters are specified deterministically using one of the strategies described in Section~\ref{sub:Estimate_alpha}. A more general setting will be considered in Section~\ref{sec:fullb}.
Three hard constraints are defined to avoid potential overfitting and ill-posedness.
\begin{constraint}
 The depth of any feature involved is less than or equal to $D$.
\end{constraint}
\begin{constraint}
 The local width of a feature is less than or equal to $L$.
\end{constraint}
\begin{constraint}
The total number of features in a model is less than or equal to  $Q$.
\end{constraint}
\noindent
The first constraint ensures that the feature space is finite, while the second and third constraints further limit the number of possible features and models. 

To put model \eqref{themodeleq} into a Bayesian framework, one has to specify priors for all parameters involved. For notational simplicity, we use $p(\cdot)$ to denote a generic prior with its arguments specifying which parameters we consider.
The structure of a specific model is uniquely defined by the vector $\M = (\gamma_1,\dots,\gamma_q)$. We start with defining the prior for $\M$ by
\begin{align}
p(\M)\propto\ &\text{I}\left(|\M|\leq Q\right)\prod_{j=1}^q \rho(\gamma_j).\label{eq:modelprior}
\end{align}
Here, $|\M|=\sum_{j=1}^q\gamma_j$ is the number of features included in the model and $Q$ as specified above is the maximum number of features allowed per model. The factors $\rho(\gamma_j)$ are introduced to give smaller prior probabilities to more complex features. Specifically, we consider 
\begin{align}
\rho(\gamma_j)=\ &  a^{\gamma_jc(F_j(\cdot,\bm{\alpha}_j))}\label{glmgammaprior}
\end{align}
with $0<a<1$ and ${c(F_j(\cdot,\bm{\alpha}_j))}\ge 0$ being a non-decreasing measure for the complexity of the corresponding feature.  In case of $\gamma_j =  0$ it holds that  $\rho(\gamma_j) = 1$ and thus the prior probability for model $\M$ only consists of the product of $\rho(\gamma_j)$ for features included in the model.  
It follows that if $\M$ and $\M'$ are two vectors only differing in one component, say $\gamma_j'=1$ and $\gamma_j = 0$, then
\[
\frac{p(\M')}{p(\M)}=a^{c(F_j(\cdot,\bm{\alpha}_j))}<1
\]  
showing that larger models are penalized more. This result easily generalizes to the comparison of more different models and provides the basic intuition behind the chosen prior. 

The prior choice implies a distribution for the model size $|\M|$.  For $Q=q$ and a constant complexity value on all features, $|\M|$ follows a binomial distribution. With varying complexity measures,  $|\M|$ follows the \textit{{Poisson binomial}} distribution \cite{wang1993number} 
which is a unimodal distribution with
$E\left\{|\M|\right\}=\sum_{j=1}^q p_j$ and
$\text{Var}\left\{|\M|\right\}=\sum_{j=1}^qp_j(1-p_j)$
where $p_j=a^{c(F_j(\cdot,\bm{\alpha}_j))}/(1+a^{c(F_j(\cdot,\bm{\alpha}_j))})$. A truncated version of this distribution is obtained for $Q<q$.

The choices of $a$ and the complexity measure $c(F_j(\cdot,\bm{\alpha}_j))$ are crucial for the quality of the model prior. 
Choosing for example
 $a = e^{-1}$ and $c(F_j(\cdot,\bm{\alpha}_j)) = \log q_{d_j}$ (with $d_j$ being the depth of $F_j$)
 would for $\gamma_j = 1$ result in
$$
  a^{c(F_j(\cdot,\bm{\alpha}_j))} = 
  \frac{1}{q_{d_j}}  \;.
$$
Therefore, the multiplicative contribution of a specific feature with depth $d$ to the model prior will be indirectly proportional to the total number of features $q_{d}$ having the depth of $d$. As we have shown in Section~\ref{sec:feature}, $q_{d}$ is rapidly growing with the feature depth $d$. Therefore, this choice gives smaller prior probabilities for more complex features. The resulting penalty closely resembles the Bonferroni correction in multiple testing as discussed for example by \citeA{BGT08} in the context of modifications of the BIC and by~\citeA{scott} in a more general setting. Such a prior construction was suggested in \citeA{hubin2018novel} for Bayesian logic regressions. However, for BGNLM, the number $q_{d}$ involves nontrivial recursions and will in practice be difficult to compute. Even obtaining good approximations would be quite challenging and investing CPU time to compute these might result in drastic deterioration in the speed of inference. 

Instead of computing $q_{d}$, we consider an alternative based on the geometric distribution which was suggested in the context of Bayesian logic regression by \citeA{Fritsch1}. For a given logic tree, they simply penalize the total number of leaves. Logic regression is a special case of BGNLM with features constructed from logic combinations of binary input variables (leaves). In the BGNLM framework, the total number of leaves corresponds to the number of algebraic and logical operations involved.  We generalize this idea further and use the operations count $oc$ of a feature as a complexity measure for BGNLM. It can be seen in Table \ref{tab:list_features} that $oc$ is a rather parsimonious property of a feature which grows smoothly as its complexity is increased. 

There remains a question of how to choose the parameter $a$.  \citeA{Fritsch1} do not give any theoretical considerations but use $a = 1/2$ or $a = 1/\sqrt 2$ in their simulation settings. For the applications of Section \ref{section5}, we will use $a=e^{-2}$ when we are concerned with prediction and $a=e^{-\log n}$ when we are  concerned with model identification. These specific choices are inspired by modifications of AIC and BIC, respectively, which are designed to control the family-wise error rate (see \citeA{bogdan2020identifying} for further details).

To complete the Bayesian model, one needs to specify priors for the components of $\bm{\beta}$ for which $\gamma_j = 1$ and, if necessary, for the dispersion parameter $\phi$. 
\begin{align}
\bm{\beta},\phi|\M\sim&p(\bm{\beta}|\M,\phi)p(\phi|\M).\label{eq:prior.par}
\end{align}
Here, $\bm\beta$ are the  regression parameters given  model $\M$. For these parameters, mixtures of $g$-priors are known to have numerous desirable properties for Bayesian variable selection and model averaging~\cite{li2018mixtures}, but simpler versions, such as  Jeffreys prior, can also be considered. Prior distributions on $\bm{\beta}$ and $\phi$ are usually defined in a way to facilitate efficient computation of marginal likelihoods  (for example, by specifying conjugate priors). They should be carefully chosen for the applications of interest. Specific choices are described in Section \ref{section4} when considering different real data sets.

As described in Section \ref{sub:Estimate_alpha}, the  $\bm{\alpha}_j$ parameters are deterministically specified during the feature generating process (a priori)  and are not considered here as model parameters. A fully Bayesian approach for all parameters of BGNLM, including $\bm{\alpha}$,  is described in Section~\ref{sec:fullb}.

\section{Bayesian Inference}\label{section3}

Posterior marginal probabilities for the model structures are given by
\begin{equation}\label{PMP}
p(\M|\bm y)
=\frac{p(\M)p(\bm y|\M)}
      {\sum_{\M'\in\MS}
      p(\M')p(\bm y|\M')},\; 
\end{equation}
where $p(\bm y|\M)$ denotes the marginal likelihood of $\bm y$ given a specific model $\M$ and $\MS$ is the model space. The marginal inclusion probability for a specific feature $F_j(\bm x,\bm\alpha_j)$ can be derived from $p(\M|\bm y)$ through
\begin{equation}\label{marginal_inclusion}
p({\gamma}_{j}=1|\bm y) =  \sum_{\M:\gamma_j=1}p(\M|\bm y),
\end{equation}
while the posterior distribution of any statistic $\Delta$ of interest becomes
\begin{equation}
p(\Delta|\bm y) =  \sum_{\M \in\MS}{p(\Delta|\M,\bm y)p(\M|\bm y)} \; .
\end{equation}
Due to the huge size of $\MS$, it is not possible to calculate the sum in the denominator of~\eqref{PMP} exactly. A standard approach  is to construct an MCMC algorithm working on the combined space of parameters and models through a reversible jump MCMC algorithm~\cite{Green} where posterior marginal probabilities can be estimated through frequencies of visits. Constructing an efficient algorithm within such a framework is, however, notoriously difficult. 

Assuming the likelihoods $p(\bm y|\M)$ are available, a much more efficient approach is to utilize~\eqref{PMP} more directly.  
In this section, we will discuss an algorithmic approach for performing the calculations needed. The main tasks are (i) to calculate the marginal likelihoods
$p(\bm y|\M)$ for a given model and (ii) to search through the model space $\MS$. 
Based on the output from this algorithm, $p(\M|\bm y)$ is estimated according to
\begin{equation}
\widehat{p}(\M|\bm y) =  
\frac{p(\M)\hat p(\bm y|\M)}
       {\sum_{\M' \in \MS^*}{p(\M')\hat p(\bm y|\M')}}\,\text{I}(\M\in\MS^*), \; 
 \label{approxpost}
\end{equation}
where two approximations are applied: a suitable subset $\MS^* \subset \MS$  is used to estimate the denominator of~\eqref{PMP}, and 
$\hat p(\bm y|\M)$ is an approximation of $p(\bm y|\M)$ as described in Section~\ref{sec:marg}.
Low values of $p(\M)\hat p(\bm y|\M)$ induce both  low values of the numerator and small contributions to the denominator in~\eqref{PMP}, hence models with low mass $p(\M)\hat p(\bm y|\M)$ will have no significant influence on posterior marginal probabilities for other models. On the other hand, models with high values of $p(\M)\hat p(\bm y|\M)$ are important. It might be equally important to include \emph{regions} of the model space where no single model has a particularly large mass but there are many  models giving a moderate contribution.  Such regions of high posterior mass are particularly important for constructing a reasonable subset $\MS^* \subseteq  \MS$ as missing them can dramatically influence our posterior estimates. 
The set of models included in $\MS^*$ will be similar to those visited in a standard application of (reversible jump) MCMC but the posterior estimates are now based on the (estimated) likelihoods instead of the frequencies of visits.
Hence, each model only needs to be visited once.

\subsection{Calculation of the Marginal Likelihoods}\label{sec:marg}
An important task is to compute the integral 
\begin{align} \label{MarginalPosterior}
p(\bm y|\M)
= & \int_{{\Theta}_\M}p(\bm y|\bm{\theta},\M)p(\bm{\theta}|\M)d\bm{\theta},
\end{align}
where 
the vector of parameters $\bm\theta$  for a specified model $\M$ consists of the regression coefficients $\{\beta_j,j:\gamma_j=1\}$ for the features included and, possibly, the dispersion parameter $\phi$. 
Assuming that the $\bm\alpha_j$'s are fixed (like in the three strategies described in Section \ref{sub:Estimate_alpha}) the BGNLM~\eqref{themodeleq} becomes a standard GLM  in which either exact calculations of the marginal likelihoods are available~\cite[linear models with conjugate priors]{Clyde:Ghosh:Littman:2010}, or numerical approximations such as simple Laplace approximations~\cite{tierney1986accurate} or integrated nested Laplace approximations \cite{rue2009eINLA} can be efficiently applied.

\subsection{Search Through the Model Space}\label{sec:gmjmcmc}

The main challenges for efficient search through the model space are both the huge number of models and the presence of many local modes. Our algorithmic approach to this is to iteratively consider subsets of models and perform an efficient search within these subsets. In particular, subsets are generated using the methods inspired by ideas from genetic programming, while search within subsets is based on a recently published mode jumping MCMC (MJMCMC) algorithm \cite{hubin2018mode}.

The MJMCMC procedure was originally proposed by~\citeA{Tjelmeland99modejumping} for problems with continuous parameters and extended to model selection by~\citeA{hubin2018mode}. The MJMCMC is described in full detail in Algorithm~\ref{MJMCMCalg0}.
\begin{algorithm}[h]
\small
\caption{\label{MJMCMCalg0}MJMCMC, one iteration from current model $\M$.}
\begin{algorithmic}[1]
\State  Generate a large jump  $\M_0^*$ according to a  proposal distribution  $q_l(\M_{(0)}^*|\M)$.
\item Perform a local  optimization, defined through $\M_{(1)}^*\sim q_o(\M_{(1)}^*|\M_{(0)}^*)$.
\State Perform a small randomization  to generate the proposal $\M^*\sim q_r(\M^*|\M_{(1)}^*)$.
\item  Generate backwards auxiliary variables $\M_{(0)}\sim q_l(\M_{(0)}|\M^*)$,  $\M_{(1)}\sim q_o(\M_{(1)}|\M_{(0)})$.
\State Set
\[\M'=
\begin{cases}\M^*&\text{with probability $r_{mh}(\M,\M^*;\M_{(1)},\M_{(1)}^*)$;}\\
\M&\text{otherwise,}
\end{cases}
\]
where
\begin{equation}
r_{mh}^*(\M,\M^*;\M_{(1)},\M_{(1)}^*) = \min\left\{1,\frac{\pi(\M^*)q_r(\M|\M_{(1)})}{\pi(\M)q_r(\M^*|\M_{(1)}^*)}\right\}\label{locmcmcgen00}.
\end{equation}
\end{algorithmic}
\end{algorithm}
The basic idea is to make a large jump (including or excluding many model components) combined with local optimization within the discrete model space to obtain a proposal model with high posterior probability. Within a Metropolis-Hastings setting, a valid acceptance probability is constructed using symmetric backward kernels, which guarantees that the resulting Markov chain is ergodic and has the desired limiting distribution. For more details, see~\citeA{hubin2018mode}. 
The use of this algorithm requires the specification of two random steps: The large jump described by $q_l$ and the small randomization described by $q_r$. These distributions are defined through probabilities of changing an included feature to be excluded and vice versa, with a large probability within $q_l$ and a small probability within $q_r$. For the local optimization operation $q_o$, combinatorial optimization algorithms are applied, where we randomly choose between simulated annealing, greedy hill-climbing, and multiple try MCMC used as an optimization tool, see \citeA{hubin2018mode} for more details.

The MJMCMC algorithm  requires that all of the features defining the model space are known in advance and are all considered at each iteration. For BGNLMs, a major problem is that it is not possible to fully specify the space $\MS$ in advance (let alone storing all potential features). We, therefore, make use of a genetically modified MJMCMC algorithm (GMJMCMC), which was originally introduced in the context of logic regression by~\citeA{hubin2018novel}. The idea behind GMJMCMC is to apply the MJMCMC algorithm on a relatively small subset  of the model space $\MS$  in an iterative setting.  We start with an initial set of features $\mathcal{S}_0$  and then iteratively update the population of features $\mathcal{S}_t$, thus generating a sequence of so-called \emph{populations}  $\mathcal{S}_1,\mathcal{S}_2,...,\mathcal{S}_T$. Each  $\mathcal{S}_t$ is a set of $s$ features and forms a separate \textit{search space} for exploration through MJMCMC iterations.
Populations dynamically evolve, allowing GMJMCMC to explore different parts of the full model space.
Algorithm~\ref{gMJMCMCalg} 
summarizes the procedure.
\begin{algorithm}[h]
\small
\caption{GMJMCMC}\label{gMJMCMCalg}
\begin{algorithmic}[1]
\State Initialize $\mathcal{S}_{0}$
\State Run the MJMCMC algorithm within the search space $\mathcal{S}_0$ for $N_{init}$ iterations and use results to initialize $\mathcal{S}_1$.
\For{$t=1,...,T-1$} 
\State Run the MJMCMC algorithm within the search space $\mathcal{S}_t$ for $N_{expl}$ iterations.\State Generate a new population $\mathcal{S}_{t+1}$
\EndFor
\State Run the MJMCMC algorithm within the search space $\mathcal{S}_{T}$ for $N_{final}$ iterations.
\end{algorithmic}
\end{algorithm}
Details of the initialization and transition between feature populations are described below and were designed specifically for BGNLM. 
All models visited, including the auxiliary ones (used by MJMCMC to generate proposals), will be included in $\MS^*$. For these models, computed marginal likelihoods will also be stored, making the costly likelihood calculations only necessary for models that have not been visited before.

\paragraph{Initialization}
Define $\mathcal{F}_0$ to be the collection of the $m$ original covariates.
The algorithm is initialized by first applying marginal testing for the selection of a subset of size $q_0 \leq m$ from $\mathcal{F}_0$.
We denote these preselected components 
by $\mathcal{S}_0 = \{{x}_1,...,{x}_{q_0}\}$ 
where for notational convenience we have ordered indices according to preselection.
MJMCMC is then run for a given number of iterations $N_{init}$  on $\mathcal{S}_0$.

\paragraph{Transition between populations}
Members of the new population $\mathcal{S}_{t+1}$ are generated by two steps.
First, in a so-called \textit{filtration} step, a  subset of the current components is selected. Components having marginal probabilities  within the search space $\mathcal{S}_t$  above a selected threshold $\rho_{del}$ are always kept (corresponding to the reproduction step in genetic programming). Features with marginal inclusion probabilities below the threshold $\rho_{del}$ are kept with probabilities proportional to their marginal inclusion probabilities. 

For those components that are removed, replacements are made with new features.
The transformations defined in Section~\ref{sec:feature} are applied to generate more complex features, where they operate on  features selected from $\mathcal{S}_t\cup\mathcal{F}_0$. Each replacement is generated randomly by the projection transformation with probability $P_{pr}$, by the modification transformation with probability $P_{mo}$, by the multiplication transformation with probability $P_{mu}$ or by a new input variable with probability $P_{in}$, where $P_{pr}+P_{mo}+P_{mu}+P_{in} = 1$. If a modification or a projection is applied, a nonlinearity is chosen from $\mathcal{G}$ with  probabilities $P_\mathcal{G}$.  If the newly generated feature is already present in $\mathcal{S}_t$ or is linearly dependent with some currently present features, it is not considered for $\mathcal{S}_{t+1}$. The implemented algorithm offers the option that a subset of $\mathcal{S}_0$  is always kept in the population throughout the search. Furthermore, one can exclude modifications, multiplications or projections by setting  $P_{mo}$, $P_{mu}$ or $P_{pr}$ to 0, respectively. By excluding projections one substantially reduces the space of models, but on the other hand, the model is put into a fully Bayesian framework, which might be of high importance for some applications.

The following result is concerned with the consistency of probability estimates from the GMJMCMC algorithm. Similarly to standard MCMC algorithms, it guarantees exact answers when the number of iterations increases given that some regularity conditions are met. 

\begin{theorem}\label{th:GMJMCMC}
Let $\MS^*$ be the set of models visited through the GMJMCMC algorithm.   Define $M_{S_t}$ to be the set of models visited within search space $\mathcal{S}_t$. Assume that exact marginal likelihoods $p(\bm y|\M)$ are available and used. Also, assume that $|\mathcal{S}_t|=s\ge \min\{Q,L\}$.
Then the model estimates based on~\eqref{approxpost} will converge to the true model  probabilities as the number of iterations $T$ goes to $\infty$.
\end{theorem}
The proof of this result is given in~\citeA{hubin2018novel}. The essential idea is that both the sets of search spaces and the models visited within search spaces are Markov chains.
The combinations of the different transformations and transitions between search spaces of GMJMCMC fulfill the requirement on irreducibility, which guarantees asymptotic exploration of the full model space of BGNLM as $\MS$ is finite and countable.

\paragraph{Remark} The result of Theorem \ref{th:GMJMCMC} relies on the exact calculation of the marginal likelihood $p(\bm y|\M)$. Apart from the linear model, the calculation of $p(\bm y|\M)$ is typically based on an approximation, giving similar approximations to the model probabilities. How precise these approximations are will depend on the type of method used. The current implementation includes Laplace approximations, integrated Laplace approximations, and integrated nested Laplace approximations. In principle, other methods based on MCMC outputs \cite{chib1995marginal,chib2001marginal} could  be incorporated relatively easily. However, these methods result in longer runtimes. In any case, all models will be visited in the end if the chain is run for long enough.

\paragraph{Parallelization strategy}
Due to our interest in quickly exploring as many \emph{unique}  high-quality models as possible, it is beneficial to run multiple chains in parallel. The process can be embarrassingly parallelized into $B$ chains. If one is mainly interested in model probabilities, then equation~\eqref{approxpost} can be directly applied with $\MS^*$ being the set of unique models visited in all runs. A  more memory efficient alternative is to utilize the following posterior estimates based on weighted sums over individual runs:\
\begin{equation}\label{weighted_sum}
\hat{p}(\Delta | \bm y) = \sum_{b=1}^B u_b\hat{p}_b(\Delta | \bm y).
\end{equation}
Here $\{u_b\}$ is a set of arbitrary normalized weights and $\hat{p}_b(\Delta |\bm y)$ are the posteriors obtained with equation~\eqref{approxpost} from run $b$ of GMJMCMC. Due to the irreducibility of the GMJMCMC procedure, it holds that $\lim_{T\rightarrow \infty}\hat{p}(\Delta| \bm y) = p(\Delta| \bm y)$ where $T$ is the number of iterations within each run. Thus, for any set of normalized weights, the approximation $\hat{p}(\Delta| \bm y)$ converges to the true posterior probability ${p}(\Delta| \bm y)$ and one can, for example, use  $u_b = 1/B$. Uniform weights, however, have the disadvantage of potentially giving too much weight to posterior estimates from  chains that have poorly explored the model space. In the following heuristic improvement, $u_b$ is chosen to be proportional to the posterior mass detected by run $b$, 
\begin{align*}
u_b=&\frac{\sum_{\M \in {\MS_{b}^{*}}} \hat p(\bm y| \M) p(\M)}{\sum_{b'=1}^B\sum_{\M' \in {\MS_{b'}^{*}} } \hat p(\bm y| \M') p(\M')}\; .
\end{align*}
This choice  indirectly penalizes chains that cover smaller portions of the model space. When estimating posterior probabilities using these weights, we only need  to store the following quantities for each run: $\hat{p}_b(\Delta| \bm y)$ for all statistics $\Delta$ of interest and $s_b = \sum_{\M' \in {\MS_b^{*}}} \hat p(\bm y|\M') p(\M')$. There is no further need for data transfer between processes. A proof that this choice of weights gives consistent estimates of posterior probabilities is given in the supplementary materials \url{https://projecteuclid.org/journals/supplementalcontent/10.1214/18-BA1141/suppdf_1.pdf} to \citeA{hubin2018novel}.

\section{Extensions of BGNLM}\label{extensions}

Two extensions of the standard BGNLM are presented in this section. We first show how the model can be extended to include latent Gaussian variables, which can be used, for example, to model spatial and temporal correlations and/or over-dispersion. Then, a fully Bayesian extension of the model is presented, where the internal parameters $\bm\alpha_j$ of the features are treated properly as parameters of the BGNLM model.

\subsection{Bayesian Generalized Nonlinear Mixed Models}\label{sec:latent}

So far, the BGNLM model has been developed under the assumption that all observations are conditionally independent. An important 
extension is to include latent variables, both to account for correlation structures and over-dispersion. This is achieved by replacing~\eqref{DeepModel}
with
\begin{align}
   \mathsf{h}(\mu) = & \beta_0  + \sum_{j=1}^{q} \gamma_{j}\beta_{j}F_{j}(\bm{x},\bm\alpha_j)+\sum_{k=1}^{r} \lambda_{k}\delta_{k}(\bm x),\label{DeepModel2}
\end{align}
where, for each $k$, $\{\delta_k(\bm x)\}$ is a zero-mean Gaussian process with covariance function $\Sigma_k(\bm x,\bm x')$.
The resulting Bayesian generalized nonlinear mixed model (BGNLMM) includes $q + r$ possible components, where $\lambda_k$ indicates whether the corresponding latent variable is to be included in the model.
The latent Gaussian variables  allow the description of different correlation structures between individual observations (e.g. autoregressive models). The covariance functions typically depend only on a few parameters, so that in practice one has $\bm{\Sigma}_k(\bm x,\bm x')=\bm{\Sigma}_k(\bm x,\bm x';\bm\psi_k)$. 

The model vector now becomes $\M= (\bm{\gamma},\bm{\lambda})$, where $\bm{\lambda}=(\lambda_1,...,\lambda_r)$.
Similarly to the restriction on the number of features that can be included in a model, we introduce an upper limit of $R$ on the number of latent variables that can be included.
The total number of models with non-zero prior probability will then be $\sum_{k=1}^{Q}\binom{q}{k} \times \sum_{l=1}^{R}\binom{r}{l}$.
The corresponding prior for the model structures is defined by
\begin{align}
p(\M)\propto\ &
\prod_{j=1}^q a^{\gamma_jc(F_j(\bm x,\bm\alpha_j))}
\prod_{k=1}^r  b^{\lambda_kv(\delta_k)}.\label{eq:modelprior2}
\end{align}
Here, the function  $v(\delta_k)\ge 0$ is a measure for the complexity of the latent variable $\delta_k$,   which is assumed to be a non-decreasing function of the number of hyperparameters defining the distribution of the latent variable. In the current implementation, we simply use the number of hyperparameters. Parameter $b$ is specified in a similar manner as  $a$.
The prior is further extended to include
\begin{align}
\bm{\psi}_k\sim&\pi_k(\bm{\psi}_k), &&\text{for each $k$ with $\lambda_k=1$}.\label{latentprior}
\end{align}
In this case, more sophisticated methods are needed to approximate the marginal likelihoods. We use the integrated nested Laplace approximation (INLA) \cite{rue2009eINLA}, but
alternative MCMC based
methods like Chib's or Chib and Jeliazkov's method \cite{chib1995marginal,chib2001marginal} are also possible. Some comparisons of these methods are presented by~\citeA{Friel2012} and~\citeA{HubinStorvikINLA}, who demonstrate that INLA performs quite well for such models.

\subsection{Fully Bayesian BGNLM}\label{sec:fullb}

For a fully Bayesian approach, the $\bm{\alpha}_j$ parameters involved in $F_j(\bm x,\bm\alpha_j)$ also have to be considered as parameters of the model. Hence, we must specify additional priors for all $\bm\alpha_j$ for which $\gamma_j = 1$. More specifically, we set
\begin{align}
\bm{\beta},\bm{\alpha},\phi|\M\sim&p(\bm{\beta}|\M,\phi)p(\bm{\alpha}|\M,\phi)p(\phi|\M).\label{eq:prior.par.full}
\end{align}
In practice, for $p(\bm{\alpha}|\M,\phi)$ we use a simple independent Gaussian prior for each component $N(0,\sigma^2_\alpha)$.
The variance $\sigma^2_\alpha$ can either be fixed or treated as a random variable. 
For computational reasons, our implementation treats $\sigma^2_\alpha$ as a constant. 

In contrast to the three  strategies described in Section~\ref{sec:feature}, the feature generation is now fully integrated into the MCMC algorithm.
When a new projection feature $F_j(\bm x,\bm \alpha_j)$ is generated, both outer parameters $\bm\alpha^{out}_j$ and the nested parameters $\bm\alpha^{in}_j$ are drawn from the prior distributions. There are no restrictions on the nonlinear transformations $g_j(\cdot)$, though  the link function $h(\cdot)$ needs to be differentiable.

This strategy is flexible and formally the most appropriate from a Bayesian perspective.  The joint space of models and parameters is, at least in principle, systematically explored, but this is computationally extremely demanding in practice. 
The joint space of models and parameters now becomes even more complex with many potential local optima of the posterior distribution.
Hence convergence typically requires a huge number of iterations.  This might be improved by drawing around the modes obtained by the previously suggested strategies, but developing and implementing this idea is a topic of further research. The current implementation of the fully Bayesian approach is only of practical use in very low-dimensional settings, but it is also an important methodological contribution which points towards the future potential of BGNLM.

The marginal likelihood under the fully Bayesian strategy becomes significantly more complicated:
\begin{align} \label{MarginalFull}
{p}(\bm y|\M)
= \int_{{\mathfrak{A}}_\M}\int_{{\Theta}_\M}p(\bm y|\bm{\theta},\bm\alpha,\M)p(\bm{\theta}|\bm\alpha,\M)p(\bm{\alpha}|\M)d\bm{\theta}d\bm{\alpha},\
\end{align}
where $\mathfrak{A}_\M$ is the parameter space for $\bm\alpha$ under model $\M$. It can be approximated using sampling via
\begin{align} \label{MarginalPosterior2}
\widehat{p}(\bm y|\M)
= \tfrac{1}{M}\sum_{t = 1}^{M}\int_{{\Theta}_\M}p(\bm y|\bm{\theta},\bm\alpha^{(t)},\M)p(\bm{\theta}|\bm\alpha^{(t)},\M)d\bm{\theta},
\end{align}
where $\bm\alpha^{(t)}, t \in \{1,...,M\}$ are drawn from the prior $p(\bm{\alpha}|\M)$. Given $\bm\alpha$, the integrals with respect to $\bm\theta$ can be calculated/approximated just as discussed previously in 
Section~\ref{sec:marg}.

\section{{Applications}}\label{section4}
In this section, we present two examples concerned with binary classification (breast cancer and spam classification), one example for prediction of a metric outcome (age prediction of abalones), and another example which focuses on obtaining interpretable models. The examples of interpretable model inference include  recovering Kepler's law from raw data, where we compare the results from BGNLM with symbolic regression, and an application of BGNLMM extension to epigenetic data. For the classification and regression tasks, the  performance of  BGNLM is compared with nine competing methods. 

Appendix~\ref{ap:furtherexamp} includes another example  concerned with asteroid classification, as well as two more examples on model inference.  The first inference example in Appendix~\ref{ap:furtherexamp} is  concerned with recovering a physical law about  planetary mass from raw data. The second one is based on simulated data from a logic regression model, where a comparison to inference using the original GMJMCMC algorithm for Bayesian logic regression \cite{hubin2018novel} is provided. 

When we refer to BGNLM in this section, we formally mean the combination of the model and the corresponding GMJMCMC algorithm for fitting the model.
In addition to the standard algorithm, a parallel version was applied\footnote{For the prediction tasks the parallel version is based on $B=32$ threads.} (denoted BGNLM{\_}PRL). %
A BGNLM with maximal depth $D=0$ is also included, which corresponds to a 
Bayesian (generalized) \emph{linear} model  using only the original covariates (denoted BGLM).
The corresponding R libraries, functions, and their tuning parameter settings are described in supplementary scripts.
Only results obtained with strategy 1 of the feature generating process are reported here. In all of the examples, $P_{\mathcal{G}}$ is uniform and the population size of GMJMCMC coincides with $Q$ and $L$. Comparisons between the different strategies for the classification examples are reported in Appendix~\ref{ap:add.ex1_3}. These comparisons demonstrate that the four different strategies behave similarly.

\subsection{Binary Classification}

\newcounter{Example} 

The performance of BGNLM is compared with the following competitive algorithms:  tree-based (TXGBOOST) and linear (LXGBOOST) gradient boosting machines, penalized likelihood (LASSO and RIDGE), deep dense neural networks with multiple hidden fully connected layers (DEEPNETS), random forest (RFOREST), na{\i}ve Bayes (NBAYES), and simple \textit{frequentist} logistic regression (LR). 
For algorithms with a stochastic component, $N=100$ runs were performed on the training data set. The test set was analyzed with each of the obtained prediction methods, where the split between training and test samples was kept fixed. We report the median as well as the minimum and maximum of the evaluation measures across those runs. For deterministic algorithms, only one run was performed. We do not report computational times since all methods are implemented using different technologies and back-end solutions, which can significantly influence the actual run time. At the same time, theoretic computational complexities of the addressed methods for model building (training phase) and prediction (test phase) are reported in Table~\ref{tab:comptime} in Appendix~\ref{ap:furtherexamp}.

We consider a BGNLM model~\eqref{themodeleq} with Bernoulli distributed observations and a logit link function.
The Bayesian model
uses the model structure prior~\eqref{eq:modelprior}  with $a=e^{-2}$. The set of nonlinear transformations is defined as  $\mathcal{G}=\{\text{gauss}(x),\text{tanh}(x)$,
$\text{atan}(x),\text{sin}(x)\}$, with $\text{gauss}(x) = e^{-x^2}$ and a uniform distribution for the selection probabilities on the functions $P_\mathcal{G}$.
The logistic regression model has a known dispersion parameter $\phi=1$ and for computational convenience the Bayesian model is completed by using Jeffreys prior on $\bm\beta$:
\begin{align*}
 p(\bm{\beta}|{\M})=&|J_n(\bm{\beta}|{\M})|^{\frac{1}{2}}\;,
 \end{align*}
where 
$|J_n(\bm{\beta}|{\M})|$ is the determinant of the  Fisher information matrix for model $\M$. The resulting posterior corresponds to performing model selection with a criterion whose penalty on the complexity is similar to a modified AIC criterion~\cite{bogdan2020identifying}.

Predictions based on the BGNLM are made according to
$$
\hat{y}_i^* = \text{I}\left(\hat p(Y_i^*=1|\bm y)\ge \eta \right),
$$
where we have used the notation $Y_i^*$ for a response variable in the test set. Furthermore,
$$
\hat p(Y_i^*=1|\bm y)=\sum_{\M \in \MS^*} \hat p(Y_i^*=1|\M,\bm y) \hat p(\M|\bm y)
$$
with $\MS^*$ denoting the set of all explored models and 
$$
\hat p(Y_i^*=1|\M,\bm y)= p(Y_i^*=1|\M,\widehat{\bm{\beta}}^{\M},\bm y),
$$
where $\widehat{\bm{\beta}}^{\M}$ is the
posterior mode in $p(\bm{\beta}|\M,\bm y)$. This results in a model averaging approach for prediction. For binary classification, we use the most common threshold $\eta  = 0.5$. Calculation of marginal likelihoods
is performed by applying the Laplace approximation. 

To evaluate the predictive performance of algorithms, we report the accuracy of predictions (ACC), false positive rate (FPR) and false negative rate (FNR). These are defined as follows:
\begin{align*}
\text{ACC} = \frac{\sum_{i=1}^{n_p}\text{I}(\hat y_i^*=y_i^*) }{n_p};\text{ }
&\text{FPR} = \frac{\sum_{i=1}^{n_p} \text{I}\left(y_i^*=0,\hat y_i^*=1\right)}{\sum_{i=1}^{n_p}  \text{I}\left(y_i^* = 0\right)};\text{ }\\
&\text{FNR} = \frac{\sum_{i=1}^{n_p}  \text{I}\left(y_i^*=1,\hat y_i^*=0\right)}{\sum_{i=1}^{n_p} \text{I}\left(y_i^*=1\right)}.
\end{align*}
Here, $n_p$ is the size of the test data sample. 

\refstepcounter{Example} \label{Ex:BreastCancer}
\subsubsection{Example \arabic{Example}: Breast cancer classification}

This example consists of breast cancer data with observations from 357 benign and 212 malignant tissues~\cite{breastdata}.
Data are obtained from digitized images of fine-needle aspirates of breast mass and can be downloaded from \url{https://archive.ics.uci.edu/ml/datasets/Breast+Cancer+Wisconsin+(Diagnostic)}.
Ten real-valued characteristics are considered for each cell nucleus: \emph{radius, texture, perimeter, area, smoothness, compactness, concavity, concave points, symmetry} and \emph{fractal dimension}. For each characteristic, the mean, standard error, and the mean of the three largest values per image were computed, resulting in 30 input variables per image. See~\citeA{breastdata} for more details. A randomly selected quarter of the images was used as a training data set, the remaining images were used as a test set. In this example, we used $D=6, L=20$ and $Q = 20$ for BGNLM.

\begin{table}[t]{
\resizebox{\textwidth}{!}{%
\begin{tabular}{llll}%
\hline 
Algorithm&ACC&FNR&FPR\\\hline
BGNLM{\_}PRL&0.9742 (0.9695,0.9812)&0.0479 (0.0479,0.0536)&0.0111 (0.0000,0.0184)\\
RIDGE&{0.9742} (-,-)&0.0592 (-,-)&{0.0037} (-,-)\\
BGLM&0.9718 (0.9648,0.9765)&0.0592 (0.0536,0.0702)&0.0074 (0.0000,0.0148)\\
BGNLM&0.9695 (0.9554,0.9789)&0.0536 (0.0479,0.0809)&0.0148 (0.0037,0.0326)\\
DEEPNETS&0.9695 (0.9225,0.9789)&0.0674 ({0.0305},0.1167)&0.0074 (0.0000,0.0949)\\
LR&0.9671 (-,-)&0.0479 (-,-)&0.0220 (-,-)\\
LASSO&0.9577 (-,-)&0.0756 (-,-)&0.0184 (-,-)\\
LXGBOOST&0.9554 (0.9554,0.9554)&0.0809 (0.0809,0.0809)&0.0184 (0.0184,0.0184)\\
TXGBOOST&0.9531 (0.9484,0.9601)&0.0647 (0.0536,0.0756)&0.0326 (0.0291,0.0361)\\
RFOREST&0.9343 (0.9038,0.9624)&0.0914 (0.0422,0.1675)&0.0361 (0.0000,0.1010)\\
NBAYES&0.9272 (-,-)&{0.0305} (-,-)&0.0887 (-,-)\\ \hline
\end{tabular}}
}
\caption{\label{t2}Comparison of performance  (ACC, FPR, FNR) of different algorithms for breast cancer data (Example \ref{Ex:BreastCancer}).  For methods with random outcomes, the median measures (with minimum and maximum in parentheses) are displayed.
The algorithms are sorted according to median accuracy.}   
\end{table}

Qualitatively, the results from Table~\ref{t2} show that the na{\i}ve Bayes classifier and random forests have the worst performance. NBAYES gives too many false positives and RFOREST too many false negatives. Tree-based boosting only marginally outperforms the random forest. All of the algorithms based on linear features are among the best performing methods, indicating that nonlinearities are not of primary importance in this data set. Nevertheless, both parallel and single-threaded versions of the GMJMCMC algorithm for BGNLM, and  also DEEPNETS, are among the best performing algorithms. BGNLM run on 32 parallel threads gives the highest median accuracy and performs substantially better than BGNLM based on only one chain.  

\refstepcounter{Example} \label{Ex:Spam}
\subsubsection{Example \arabic{Example}: Spam classification}
The second classification task uses  data from \citeA{cranor1998spam} for detecting spam emails, which can be downloaded from \url{https://archive.ics.uci.edu/ml/datasets/spambase}. The concept of ‘spam’ is extremely diverse and includes advertisements for products and websites, money-making schemes, chain letters, the spread of unethical photos and videos, etc. In this data set, the collection of spam emails consists of  messages which have been actively marked as spam by users, whereas non-spam emails consist of messages filed as work-related or personal. 
The data set includes 4601 emails, with 1813 labeled as spam. For each email, 58 characteristics are listed which can serve as explanatory input  variables. These include 57 continuous and 1 nominal variable, where most of these are concerned with the frequency of particular words or characters.  Three variables provide different measurements on the sequence length of consecutive capital letters.  The data were randomly divided into a training set of 1536 emails and a test set of the remaining 3065 emails.
The model, settings, and performance measures are the same as for the previous example, except for the hyper-parameters $Q$ and $L$ and the population size of the GMJMCMC algorithm, which are all set to 100 due to the significantly increased number of input covariates to be considered.

\begin{table}[!ht]{
\resizebox{\textwidth}{!}{%
\begin{tabular}{llll}%
\hline 
Algorithm&ACC&FNR&FPR\\\hline
TXGBOOST&0.9465 (0.9442,0.9481)&0.0783 (0.0745,0.0821)&0.0320 (0.0294,0.0350)\\
RFOREST&0.9328 (0.9210,0.9413)& 0.0814 (0.0573,0.1174)&0.0484 (0.0299,0.0825)\\
DEEPNETS&0.9292 (0.9002,0.9357)& 0.0846 (0.0573,0.1465)&0.0531 (0.0310,0.0829)\\
BGNLM{\_}PRL&0.9251 (0.9139,0.9377)&0.0897 (0.0766,0.1024)&0.0552 (0.0445,0.0639)\\
BGNLM&0.9243 (0.9113,0.9328)&0.0927 (0.0808,0.1116)&0.0552 (0.0465,0.0658)\\
LR&0.9194 (-,-)&0.0681 (-,-)&0.0788 (-,-)\\
BGLM&0.9178 (0.9168,0.9188)&0.1090 (0.1064,0.1103)&0.0528 (0.0523,0.0538)\\
LASSO&0.9171 (-,-)& 0.1077 (-,-)&0.0548  (-,-)\\
RIDGE&0.9152 (-,-)&0.1288 (-,-)& 0.0415 (-,-)\\
LXGBOOST&0.9139 (0.9139,0.9139)&0.1083 (0.1083,0.1083)&0.0591 (0.0591,0.0591)\\
NBAYES&0.7811 (-,-)&0.0801 (-,-)&0.2342 (-,-)\\
\hline
\end{tabular}}
}
\caption{\label{t3} Comparison of performance (ACC, FPR, FNR) of different algorithms for spam data (Example \ref{Ex:Spam}). See caption of Table~\ref{t2} for details.}
\end{table}

 Table~\ref{t3} reports the results for the different methods. Once again, the na{\i}ve Bayes classifier performed worst. Apart from that, the order of performance of the algorithms is quite different from the previous example. The tree-based algorithms show the highest accuracy whereas the five algorithms based on linear features have lower accuracy. This indicates that nonlinear features are important to discriminate between spam and non-spam emails in this data set. As a consequence, BGNLM performs significantly better than BGLM. Specifically, the parallel version of BGNLM provides almost the same accuracy as DEEPNETS, with the minimum accuracy over 100 runs being actually larger, whereas the median and maximum accuracy are quite comparable. Tree-based gradient boosting and random forests, however, perform substantially better, mainly since they can optimize cutoff points for the continuous variables. One way to potentially improve the performance of BGNLM would be to include multiple characteristic functions in $\cal{G}$, such as  $\text{I}(x>\mu_{x}), \text{I}(x<F^{-1}_{0.25}(x)), \text{I}(x>F^{-1}_{0.75}(x))$. This would allow the generation of features with splitting points like in random trees.

\subsubsection{Complexity of features used for binary classification}
One can conclude from these two classification examples that BGNLM has good predictive performance both when nonlinear patterns are present (Example~\ref{Ex:Spam}) or when they are not (Example~\ref{Ex:BreastCancer}).  Similar results hold for an additional classification example presented in Appendix~\ref{Ex:NeoAsteroids} where nonlinearities play no role and yet BGNLM appears on the mark. Additionally, BGNLM has the advantage that its generated features are highly interpretable.  Excel sheets are provided as supplementary material \cite{supt} and present all features detected by BGNLM with a posterior probability larger than 0.1. Table \ref{Tab:complexity} provides the corresponding frequency distribution of the complexity (operations counts) of these features.

\begin{table}[tb]
\resizebox{\textwidth}{!}{%
\begin{tabular}{cc}
\begin{tabular}{lrrr}%
\multicolumn{4}{l}{\underline{\textbf{Example 1}: Breast cancer}}\\
compl.&BGNLM&BGNLM{\_}PRL&BGLM\\
1&11.30&14.20&29.83\\
2&3.09&0.04&0.00\\
3&0.30&0.00&0.00\\
4&0.00&0.00&0.00\\
5&0.00&0.00&0.00\\
6&0.00&0.00&0.00\\
7&0.00&0.00&0.00\\
8&0.00&0.00&0.00\\
9&0.00&0.00&0.00\\
$\geq$10&0.00&0.00&0.00\\
\hline
Total&14.42&14.24&29.83\\
\hline
\end{tabular}&
\begin{tabular}{lrrr}%
\multicolumn{4}{l}{\underline{\textbf{Example 2}: Spam mail}}\\
compl.&BGNLM&BGNLM{\_}PRL&BGLM\\
1&36.34&39.87&49.83\\
2&14.45&21.47&0.00\\
3&2.83&4.24&0.00\\
4&0.69&1.36&0.00\\
5&1.15&1.56&0.00\\
6&0.92&1.24&0.00\\
7&0.37&0.57&0.00\\
8&0.25&0.33&0.00\\
9&0.04&0.16&0.00\\
$\geq$10&0.15&0.11&0.00\\
\hline
Total&57.19&71.910&49.83\\\hline
\end{tabular}
\end{tabular}
}
\caption{\label{Tab:complexity}Mean frequency distribution of feature complexities detected by the different BGNLM algorithms in 100 simulation runs for  Examples \ref{Ex:BreastCancer} and  \ref{Ex:Spam}. The final row for each example gives the mean of the total number of features in 100 simulation runs which had a posterior probability larger than 0.1.}
\end{table}

In Example \ref{Ex:BreastCancer}, the parallel version of BGNLM reported a substantially smaller number of nonlinearities than the single-threaded version.  No projections were detected, while multiplications were more often detected than modifications. Interestingly, the nonlinear features reported by the parallel versions of BGNLM consisted  only of the following two multiplications:  (standard error of the area) $\times$ (worst texture) reported 3 times by BGNLM{\_}PRL and (worst texture) $\times$ (worst concave points) reported once by BGNLM{\_}PRL. While BGLM almost always included all 30 variables in the model (in 100 simulation runs only 17 out of 3\,000 possible linear features had posterior probability smaller than 0.1), BGNLM delivered more parsimonious models.

In Example \ref{Ex:Spam}, there is much more evidence for nonlinear structures. The nonlinear features with the  highest detection frequency over simulation runs in this example were always modifications. 
For BGNLM{\_}PRL, the four modifications sin($X_{7}$), gauss($X_{36}$), atan($X_{52}$), and tanh($X_{52}$) were among the top-ranking nonlinear features. Although modifications were most important in terms of replicability over simulation runs, BGNLM also found many multiplications and projections. 
From the 3204 nonlinear features reported in 100 runs by BGNLM{\_}PRL, there were 
more than 998 which included one multiplication, 116 with two multiplications, and even 3 features with three multiplications. Furthermore, there were 353 features including one projection, 12 features with two nested projections, and even 3 features where three projections were nested. These highly complex features typically occurred only in one or two simulation runs. Despite the good performance of the parallel versions of the algorithm, it seems that even more parallel threads or longer chains might be necessary to get consistent results over simulation runs in this example.

\subsection{Prediction of Metric Outcome}

For the prediction of a metric outcome, we consider a BGNLM model~\eqref{themodeleq} with a Gaussian distribution and identity link.
The set of nonlinear transformations is now $\mathcal{G} = \{\text{sigmoid}(x)$, $\exp(-|x|)$, $\log(|x|+1)$, $|x|^{1/3}$, $|x|^{5/2},|x|^{7/2}\}$ where selection probabilities $P_\mathcal{G}$ are again uniform. Furthermore, the restrictions $D = 6$ for the depth, $L=15$ for the local width and $Q=15$ for the maximum number of features per model are applied.
We present results both for $a=e^{-2}$ and $a=e^{- \log n}$ in the prior on model structures, giving AIC and BIC-like penalties for the model complexity, respectively. The
parameter priors are specified as  
\begin{align}
 p(\sigma^2) =& \sigma^{-2} \quad \text{   and   } \quad  p(\bm{\beta}|\M,\sigma^2)=|J_n(\bm{\beta}|\M,\sigma^2)|^{\frac{1}{2}},\;\label{JefPriorNormal}
\end{align}
where $|J_n(\bm{\beta}|\M,\sigma^2)|$ is the determinant of the corresponding Fisher information matrix. Hence, the prior for the coefficients \eqref{JefPriorNormal} becomes Jeffreys prior. 
In this case, marginal likelihoods conditional on fixed values of $\bm{\alpha}$ can be computed exactly.

Here, we compare BGNLM, BGNLM{\_}PRL, and BGLM with almost the same set of competing algorithms as before in the classification examples, with two exceptions. Since the na{\i}ve Bayes classifier and logistic regression are not suitable for predicting a metric outcome, we considered instead VARBAYES, which refers to Bayesian linear regression fitted with variational Bayes \cite{carbonetto2012scalable}, and Gaussian regression (GR), which refers to simple frequentist linear regression \cite{McCullagh-Nelder-1989}.

The performance of the different methods was compared according to the root mean squared error (RMSE), the mean absolute error (MAE), and the Pearson correlation coefficient between the observed data and their predictors (CORR) which is the same as the
square root of the coefficient of determination. These measures are   defined as follows:
\begin{eqnarray*}
\text{RMSE} &= &  \sqrt{\frac{\sum_{i=1}^{n_p}(\hat y_i^*-y_i^*)^2}{n_p}}; \quad  \text{MAE} = {\frac{\sum_{i=1}^{n_p}|\hat y_i^*-y_i^*|}{n_p}}; \\ \text{CORR}&=&{\frac{\sum_{i=1}^{n_p}(\hat y_i^*-\Bar{\hat y}^*)( y_i^*-\Bar{y}^*)}{\sqrt{\sum_{i=1}^{n_p}(\hat y_i^*-\Bar{\hat y}^*)^2}\sqrt{\sum_{i=1}^{n_p}( y_i^*-\Bar{y}^*)^2}}}\;.
\end{eqnarray*}
 
Like in the case of binary classifications, for
algorithms with a stochastic component we performed $N=100$ runs on the training data set and then analyzed the test set with each of the obtained models. 

\refstepcounter{Example} \label{Ex:ShellAge}
\subsubsection{Example \arabic{Example}: Abalone shell age prediction}

The Abalone data set~\cite{nash1994population}, downloaded from \url{https://archive.ics.uci.edu/ml/datasets/Abalone}, has served as a benchmark data set for prediction algorithms for more than two decades. The aim is to predict the age of abalone from physical measurements. The input variables used are
{\it Sex} (categorical, Male/ Female/Infant), {\it Length} (continuous, longest shell measurement),
{\it Diameter} (continuous, perpendicular to length), 
{\it Height}  (continuous, with meat in shell),
{\it Whole weight} (continuous, whole abalone), 
{\it Shucked weight}  (continuous, weight of meat),
{\it Viscera weight} (continuous, gut weight, after bleeding) and
{\it Shell weight} (continuous, after being dried).
All measurements of these variables are in mm or grams.
The outcome variable, age in years, is obtained by adding 1.5 to the number of rings. The counting of rings is a tedious and time-consuming task and therefore there is some interest in predicting the age from the other measurements which are easier to obtain. For this data set, a total of 4\,177 observations are present, of which 3\,177 randomly chosen observations were used for training and the remaining 1000 observations were used for testing for all of the compared approaches.
\begin{table}[!ht]{
\resizebox{\textwidth}{!}{%
\begin{tabular}{llll}%
\hline 
Algorithm&RMSE&MAE&CORR\\\hline
BGNLM{\_}PRL (BIC)&1.9573 (1.9334,1.9903)&1.4467 (1.4221,1.4750)&0.7831 (0.7740,0.7895)\\
BGNLM (BIC)&1.9690 (1.9380,2.0452)&1.4552 (1.4319,1.5016)&0.7803 (0.7616,0.7882)\\
BGNLM{\_}PRL (AIC)&1.9720 (1.9328,2.0081)&1.4548 (1.4377,1.4903)&0.7795 (0.7693,0.7893)\\
BGNLM (AIC)&2.0046 (1.9573,2.0560)&1.4821 (1.4471,1.5209)&0.7707 (0.7566,0.7831)\\
RFOREST&2.0352 (2.0020,2.0757)&1.4924 (1.4650,1.5259)&0.7633 (0.7530,0.7712)\\
BGLM&2.0758 (-,-)&1.5381 (-,-)&0.7522 (-,-)\\
LASSO&2.0765 (-,-)&1.5386 (-,-)&0.7514 (-,-)\\
VARBAYES&2.0779 (-,-)&1.5401 (-,-)&0.7516 (-,-)\\
GR&2.0801 (-,-)&1.5401 (-,-)&0.7500 (-,-)\\
LXGBOOST&2.0880 (2.0879,2.0880)&1.5429 (1.5429,1.5429)&0.7479 (0.7479,0.7479)\\
TXGBOOST&2.0881 (2.0623,2.1117)&1.5236 (1.4981,1.5438)&0.7526 (0.7461,0.7590)\\
RIDGE&2.1340 (-,-)&1.5649 (-,-)&0.7347 (-,-)\\
DEEPNETS&2.1466 (1.9820,3.5107)&1.5418 (1.3812,3.1872)&0.7616 (0.6925,0.7856)\\\hline
\end{tabular}}
} 
\caption{\label{tAbalone}Comparison of performance  (RMSE, MAE, CORR) of different algorithms for abalone shell data  (Example \ref{Ex:ShellAge}).  For methods with random outcomes, the median measures (with minimum and maximum in parentheses) are displayed.
The methods are sorted according to median RMSE.}   
\end{table}

A more detailed description of the data set is given in~\citeA{waugh1995extending}. In the original data set, the categorical variable {\it Sex} had a fourth level called Trematode. This referred to shells being castrated due to Trematode infection, but there were only relatively few subjects of that type and hence they were removed. Interestingly, infant abalones are not necessarily younger than male or female abalones, which makes the prediction task more difficult.
Another challenge is the multicollinearity issue due to large correlations between all weight measurements and measures of length.

\begin{table}[!ht]{
\centering{
\begin{tabular}{ll|ll}%
Frequency&Linear&Frequency&Nonlinear\\
\hline 
99	& Female         & 43	& exp(ShuckedWeight) \\
99	& VisceraWeight & 26	& sigmoid(WholeWeight) \\
96	& Male           & 26	& exp(ShuckedWeight) * WholeWeight \\
86	& ShellWeight   & 16 &	Male * ShuckedWeight \\
86	& ShuckedWeight & 15 &	ShuckedWeight * ShuckedWeight \\
80	& Height         & 15 &	$\mbox{Height}^{1/3}$ \\
67	& Whole Weight   & 15 &	Female * Height \\
56	& Diameter       & 14 &	log(WholeWeight) \\
23	& Length         & 14 &	ShuckedWeight * WholeWeight \\
&& 14	&Female * WholeWeight \\
&& 13 &	log(ShuckedWeight) \\
&& 13 & Female * ShellWeight \\
\hline
\end{tabular}}
} 
\caption{\label{tAbalone2} Frequency of selection of features in 100 simulation runs of the nine input features and all nonlinear features which were detected in more than 10 runs for abalone shell age data  (Example \ref{Ex:ShellAge}). The frequency is the number of simulations that include the given feature, linear features are listed on the left, nonlinear - on the right.}   
\end{table}

The results presented in Table \ref{tAbalone} indicate that  BGNLMs with both $a=e^{-2}$ (AIC) and $a=e^{-\log n}$  (BIC) outperform all other algorithms in terms of prediction accuracy, and the parallel versions  (BGNLM\_PRL) perform even better. The worst run of BGNLM\_PRL (BIC) (with RMSE = 1.99) is still better than the best run of random forests (RMSE = 2.00) which was the third-best algorithm. Interestingly the performance of DEEPNETS was very unstable between repeated runs. It had the largest variation of RMSE, occasionally giving really good prediction results (though slightly worse than BGNLM) but in most cases worse than all the competing algorithms. 


Table~\ref{tAbalone2} provides information about the features that were most often detected by BGNLM\_PRL (i.e. had a posterior probability above 0.5) in  100 simulation runs.  The two dummy variables for {\it Sex} and the {\it Viscera Weight} were almost always selected. Apart from {\it Length}, all the other input features were selected in more than half of the simulation runs. The most important nonlinear feature was $\exp(ShuckedWeight)$ which was detected in 43 runs, followed by sigmoid($WholeWeight$) and $\exp(ShuckedWeight) * WholeWeight$. The last feature has a depth of 3, but the majority of frequently observed nonlinear features are either modifications or multiplications of input features with depth 2. Within all the simulation runs, only 6 different projections were detected, four of them only once and two of them twice.  An Excel spreadsheet is provided as supplementary material, which includes information about all detected features  with a posterior probability larger than 0.1.

\subsection{Model Inference}

Example~\ref{Ex:Kepler} and the example in Appendix~\ref{Ex:JupiterMass} 
are based on data sets describing physical parameters of newly discovered exoplanets. The data were originally collected and continues to be updated by Hanno Rein at the Open Exoplanet Catalogue GitHub (\url{https://github.com/OpenExoplanetCatalogue/}) repository~\cite{exocat}. The input covariates include planet and host star attributes, discovery methods, and dates of discovery.  We use a subset of $n = 223$ samples containing all planets with no missing values  to rediscover two basic physical laws which involve some nonlinearities. 
We compare the performance of BGNLM  when running different numbers of parallel threads. For the Kepler's third law, we also compare BGNLM to a symbolic regression approach~\cite{koza1994genetic}, implemented in the Python library \texttt{gplearn}.

For this example (and also for the examples in Appendix~\ref{Ex:JupiterMass} dealing with a planetary-mass law and Appendix~\ref{ap:sim.ex} dealing with simulated logic data),  we utilize the BGNLM model~\eqref{themodeleq} with conditionally independent Gaussian observations and the identity link.
Two different sets of nonlinear transformations, $\mathcal{G}_1 = \{\text{sigmoid}(x), \sin(x), \cos(x)$, $\tanh(x)$, $\text{atan}(x)$, $|x|^{1/3}\}$ and $\mathcal{G}_2 = \{\text{sigmoid}(x), \sin(x),\exp(-|x|),\log(|x|+1), |x|^{1/3}$, $|x|^{2.3}$, $|x|^{7/2}\}$ are considered with a uniform $P_\mathcal{G}$ in both cases. We 
 restrict the depth to $D = 5$, the local width to $L=15$ and the maximum number of features in a model to $Q=15$. $\mathcal{G}_1$ is an adaptation of the set of transformations used in the prediction examples. Adding $|x|^{1/3}$ results in a model space which includes a closed form expression of Kepler's 3rd law in Example \ref{Ex:Kepler}.
 $\mathcal{G}_2$ is a somewhat larger set where the last two functions are specifically motivated to facilitate the generation of interesting features linking the mass and luminosity of stars \cite{kuiper1938empirical,salaris2005evolution}.  
 For the prior of the model structure~\eqref{eq:modelprior}, we choose $a=e^{- \log n}$ giving a BIC like penalty for the model complexity. The
parameter priors are specified again by~\eqref{JefPriorNormal}.
In this case, marginal likelihoods conditional on fixed values of $\bm{\alpha}$ can be computed exactly.

 The focus in these examples is on correctly identifying important features. A missed feature is considered to be a more serious loss than including some extra features. Consequently, we are using a  threshold value of $0.25$ for the feature posteriors to define positive detections. To evaluate the performance of algorithms,  we  report estimates for the power (Pow), the false discovery rate (FDR), and the expected number of false positives (FP) based on $N$ simulation runs.  These measures are defined as follows:  
\begin{align*}
\text{Pow} = 
\tfrac{1}{N}\sum_{l=1}^N \text{I}(\hat\gamma^l_{j^*}=1);\quad
\text{FDR} =&\tfrac{1}{N}\sum_{l=1}^N\frac{\sum_j\text{I}(\gamma_j=0,\hat\gamma_j^l=1)}{\sum_j\text{I}(\hat\gamma_j^l=1)}; \\
\text{FP} =& \tfrac{1}{N}\sum_{l=1}^N  \sum_{j\neq j^*}\text{I}(\hat\gamma^l_j=1).
\end{align*}
Here, $\hat\gamma^l_j = \text{I}(\hat p(\gamma_j|\bm y)>0.25)$ denotes the identification of $\gamma_j$ in run $l$ of the algorithm and $j^*$ is the index of a true feature, which means a feature which is in accordance with the well known physical laws. 

\refstepcounter{Example} \label{Ex:Kepler}
\subsubsection{Example \arabic{Example}: Kepler's third law}
In this example, we want to model the semi-major axis of the orbit, \textit{SemiMajorAxisAU}, as a function of 10 potential input variables,
which are described and aliased in Table~\ref{tab:alias}.


\begin{table}[tb]
\resizebox{\textwidth}{!}{%
\begin{tabular}{llllll}
Variable&Alias&Full name&Variable&Alias&Full name\\\hline
$y$&$a$     &\textit{SemiMajorAxisAU}  \\
$x_1$&  -   & \textit{TypeFlag}&$x_2$&  $R_p$& \textit{RadiusJpt}\\
$x_3$&  $P$   & \textit{PeriodDays}&$x_4$&  $m_p$   & \textit{PlanetaryMassJpt}\\
$x_5$&  $e$   & \textit{Eccentricity}&$x_6$&  $M_h$   & \textit{HostStarMassSlrMass}\\
$x_7$&  $R_h$   & \textit{HostStarRadiusSlrRad}&$x_8$&  $Fe_h$   & \textit{HostStarMetallicity}\\
$x_9$&  $T_h$     & \textit{HostStarTempK}&$x_{10}$& $\rho_p$     & \textit{PlanetaryDensJpt}\\\hline
\end{tabular}
}
\caption{\label{tab:alias} Model variable names, their (physical) aliases, and full names used in  Examples~\ref{Ex:Kepler}~and~\ref{Ex:JupiterMass}.}
\end{table}

Kepler's third law says that the square of the orbital period $P$ of a planet is directly proportional to the cube of the semi-major axis $a$ of its orbit. Mathematically, this can be expressed as
\begin{align}
\frac{P^2}{a^3}=\frac{4\pi^2}{G(M+m)}\approx\frac{4\pi^2}{GM},\label{3kepl0}
\end{align}
where $G$ is the gravitational constant, $m$ is the mass of the planet, $M$ is the mass of the corresponding hosting star, and $M\gg m$. The  approximation on the right-most side of~\eqref{3kepl0} is due to neglecting $m$. 
Equation~\eqref{3kepl0} can be reformulated as
\begin{align}
a \approx K\left(P^2 M_h \right)^{1/3}.\label{3kepl1}
\end{align}
The mass of the hosting star $M_h$ is measured in units of Solar mass. Thus the constant $K$ includes not only the gravitational constant $G$ but also the normalizing constant for the mass.  There exist certain power laws which relate the mass $M_h$ of a star with its radius $R_h$  as well as with its temperature $T_h$. Although these relationships are not linear, it is still not particularly surprising that there are two features which are strongly correlated with the target feature, namely $(P^2 R_h )^{1/3}$ (with a correlation of 0.9999667) and $(P^2T_h )^{1/3}$ (with a correlation of 0.9995362), both of which are also treated as true positives in our study.

To assess the ability of BGNLM to detect these  features we performed $N=100$ runs for both $\mathcal{G}_1$ and $\mathcal{G}_2$ when using 1, 16, and 64 threads, respectively. In each of the threads, the algorithm was first run for 10\,000 iterations, generating new populations at every 250th iteration. Then a larger number of iterations was based on the last population, which was run until a total of 10\,000 unique models were obtained from it. The results for BGNLM are  presented in the upper part of Table~\ref{Tab:Results_Kepler}. A detection of any of the three highly correlated features described above is counted as a true positive, other features are counted as false positives, and the definitions of Pow and FDR are modified accordingly.
\begin{table}[!ht]
\resizebox{\textwidth}{!}{%
\begin{tabular}{lcccccc|cccccc}
\hline
&\multicolumn{6}{c|}{BGNLM{\_}PRL using $\mathcal{G}_1$}&\multicolumn{6}{c}{BGNLM{\_}PRL using $\mathcal{G}_2$}\\\hline 
Threads& $F_1$& $F_2$& $F_3$ &Pow& FP&FDR&  $F_1$& $F_2$& $F_3$ &Pow& FP&FDR\\\hline
64& 81& 71&1 &1.00& 0.02&0.01&72&71&3&0.99&0.04&0.015\\
16& 34& 41&32&0.84&0.46&0.18 &39&42&13&0.83&0.55&0.22\\
 1&    6& 5 &3 &0.141&0.65&0.86&7&4&3&0.14&1.81&0.86\\\hline\hline
&\multicolumn{6}{c|}{SymbolicRegressor using $\bm x_1$}&\multicolumn{6}{c}{SymbolicRegressor using $\bm x_2$}\\\hline 
  & $F_1$& $F_2$& $F_3$ &Pow& FP&FDR&  $F_1$& $F_2$& $F_3$ &Pow& FP&FDR\\\hline
 &36&0&0&0.36&0.64&0.64
 &20&0&0&0.20&0.80&0.80\\
 \hline
\end{tabular}
}
\caption{\label{Tab:Results_Kepler}Results for detecting Kepler's third law as in Equation \eqref{3kepl0} based on the decision rule that the posterior probability of a feature  is larger than $0.25$  (Example \ref{Ex:Kepler}). The three features  $\left(P^2 M_h\right)^{1/3}$,  $\left(P^2 R_h \right)^{1/3}$  and  $\left(P^2 T_h \right)^{1/3}$ are all counted as true positives, all other selected features - as false positives. Apart from the power to detect each of these features ($F_1, F_2$ and $F_3$) we report Pow, FP and FDR.  BGNLM is applied using the nonlinear sets (NL set) $\mathcal{G}_1$ and $\mathcal{G}_2$ and different numbers of parallel threads. \texttt{SymbolicRegressor} is applied on the functions 'add', 'sub', 'mul', 'div' and $x^{1/3}$ with a population size of 40\,000 and 50 generations and two subsets of input variables, $\bm x_1=(x_3,x_6,x_7,x_9)$ and $\bm x_2=(x_3,x_5,x_6,x_7,x_8,x_9)$.
For both methods, the procedures are repeated 100 times.} 
\end{table}

With increasing computational effort (number of threads), the  power of recovering the true physical law in a closed-form is converging to 1 and FDR is getting close to 0 for BGNLM. In this example, there is not such a big difference between the nonlinear sets $\mathcal{G}_1$ and $\mathcal{G}_2$. Note that these results were obtained with a fairly small sample size of $n = 223$ observations. 
Sections~\ref{Ex:JupiterMass} and Appendix~\ref{ap:sim.ex}    provide two more examples where BGNLM gives highly interpretable results.

\paragraph{Comparison with symbolic regression} We used the same data to identify the underlying mathematical expression using the \texttt{SymbolicRegressor} routine within the Python library \texttt{gplearn} (\url{https://gplearn.readthedocs.io/en/stable/}). We were not able to obtain
reasonable results using a set of generative functions similar to the sets $\mathcal{G}_1$ and $\mathcal{G}_2$ used for BGNLM. We, therefore, reduced the set of generative functions to  'add', 'sub', 'mul', 'div' and $x^{1/3}$ (note that the 'add' and 'sub' functions are included in the projection transformation in BGNLM, the 'del' function was included in neither $\mathcal{G}_1$ nor $\mathcal{G}_2$ but removing this function gave worse results using the \texttt{SymbolicRegressor} routine).
Even with this reduced set of functions, symbolic regression was not able to give meaningful models when using all input variables. Therefore, we also reduced the set of input variables and present the results for  two subsets, $\bm x_1=(x_3,x_6,x_7,x_9)$ corresponding to \emph{PeriodDays}, \emph{HostStarMassSlrMass}, \emph{HostStarRadiusSlrRad} and \emph{HostStarTempK} and $\bm x_2=(x_3,x_5,x_6,x_7,x_8,x_9)$, which additionally includes \emph{Eccentricity} and \emph{HostStarMetallicity}. The results are based on 50 generations with a population size of $40\,000$ within the genetic programming routine this method is based on.

Symbolic regression results in only one feature as output, and, for that reason, it always holds that FDR=FP=1-Pow.
The parallel versions of BGNLM perform much better on all measures provided, although a much wider range of functions and input variables was considered. Whenever \texttt{SymbolicRegressor} finds a model that we considered as true positive, it was always the actual target feature $\left(P^2 M_h\right)^{1/3}$ and never one of the two correlated features $\left(P^2 R_h \right)^{1/3}$  or  $\left(P^2 T_h \right)^{1/3}$.
While BGNLM automatically gives uncertainty measures for features/expressions to be included, the  \texttt{SymbolicRegressor} needs to be run several times in order to obtain similar uncertainty measures. The runtime for the
 \texttt{SymbolicRegressor} with 100 repeats was comparable (somewhat larger than) BGNLM for this example.

\paragraph{Interpretability of BGNLM results} \label{ap:interp}

The key feature of BGNLM which allows to obtain interpretable models is that there is a set $\mathcal{G}$ of nonlinear transformations and hence feature generation becomes highly flexible. To illustrate the importance of the choice of $\mathcal{G}$, we reanalyze Example~\ref{Ex:Kepler} on Kepler's third law with BGNLM{\_}PRL using only the sigmoid function as nonlinear transformation. We also consider different restrictions on the search space:
\begin{itemize}
\item  1. $\mathcal{G}=\{\text{sigmoid}(x)\}$, $D = 5$;
\item  2. $\mathcal{G}=\{\text{sigmoid}(x)\}$, $D = 300$, and multiplication probability $P_{mu}=0$;
\item  3. $\mathcal{G}=\{\text{sigmoid}(x)\}$, $D = 300$,  $P_{mu}=0$ and $p(\gamma_j) \propto 1$.
\end{itemize}
For these settings, it is not possible to obtain the correct model in a closed form, but 
Kepler's 3rd law can still be well approximated.
In the first setting, the true model is infeasible since the cubic root function is not a part of $\mathcal{G}$ but the multiplication of features is still possible. In the second setting, multiplications are not allowed. On the other hand, there is no longer any feasible \textit{hard} restriction on the  depth of features ($D = 300$). Finally, in the third setting, all features get a uniform prior in the feature space, disregarding complexity. As a consequence of the lack of regularization, we expect that highly complex features are generated.

\begin{table}[!ht]
\resizebox{\textwidth}{!}{%
\begin{tabular}{cl|cl|cl}
\multicolumn{2}{c|}{Setting 1}&\multicolumn{2}{c|}{Setting 2}&\multicolumn{2}{c}{Setting 3}\\\hline
Fq&Feature&Fq&Feature&Fq&Feature\\\hline
99&$x_3$&100&$x_3$&100&$x_3$
\\
98&$x_3^2$&72&$\sigma$(-10.33+0.24$x_4$-8.83$x_8$)&54&$x_2$
\\
93&$x_3x_{10}$&64&$x_{10}$&21&$\sigma$(-16.91-4.94$x_2$)
\\
4&$x_3^2x_{10}$&62&$x_2$&19&$x_9$
\\
1&$x_3x_9$&16&$\sigma$(0.21+0.01$x_3$+0.20$x_7$)&16&$x_5$
\\
1&$x_3^2x_9$&9&$x_4$&14&$x_{10}$
\\
1&$x_3x_{10}^2$&7&$\sigma$(-13.11-7.76$x_8$-3.33$x_2$+0.40$x_{10}$)&10&$\sigma$(6.88$\times10^9$-3.92$x_2$+\\&&&&&3.44$\times10^9$$\sigma$(-13.57-0.17$x_4$-\\&&&&&2.84$x_2$-7.66$x_8$+0.54$x_{10}$)\\&&&&&-13.76$\times10^9$$\sigma$($\sigma$(-13.57-\\&&&&&0.17$x_4$-2.84$x_2$-7.66$x_8$+\\&&&&&0.54$x_{10}$)))
\\
1&$x_3^2x_7$&5&$\sigma$(-3.36+2.83$x_3$+0.21$x_3$-3.36$x_9$)&9&$x_4$
\\
1&$x_3^2x_6$&3&$\sigma$($\sigma$(-10.33+0.24$x_4$)-8.83$x_8$)&8&$\sigma$(-13.57-0.17$x_4$-\\&&&&&2.84$x_2$-7.66$x_8$+0.54$x_{10}$)
\\
1&$x_3^3$&3&$\sigma$(0.15+0.05$x_4$-0.01$x_3$+0.15$x_7$)&7&$\sigma$(0.21+0.21$x_3$)
\\
0&Others&4&Others&$>300$&Others\\
\hline
\end{tabular}
}
\caption{\label{dnn1}The ten most frequent features detected under Settings 1, 2, and 3 (Example \ref{Ex:Kepler}).}
\end{table}
Table~\ref{dnn1} illustrates the effects of these changes on the interpretability of models. We report the ten most frequently detected features over $N = 100$ simulations. The results  are not too surprising. Restricting the set of nonlinear transformations results in increasingly more complex features. In Setting 1, there is not a single occurrence of a sigmoid function, while in Setting 2 the feature  $\sigma$(-10.33+0.24$x_4$-8.83$x_8$) is selected in almost  75\% of the runs. Removing the complexity penalty in Setting 3 yields highly complex features, which are however no longer replicable over simulation runs. We conclude that more flexible sets of nonlinear transformations $\mathcal{G}$ allow interpretable models to be selected. These models also have a similar predictive performance to complex models based on a less flexible set of transformations. Problems with the latter approach include overfitting and a substantially higher need for memory and computational requirements, at least in the prediction stage. In contrast, BGNLM used with proper parameter settings constructs nonlinear models that achieve a state-of-the-art prediction performance, while remaining relatively simple. Hence, they represent sophisticated phenomena in a fairly parsimonious way.

\refstepcounter{Example} \label{ap:ex.epi}
\subsubsection{Example \arabic{Example}: Epigenetic data with latent Gaussian variables}

This example illustrates how the extended BGNLMM model~\eqref{DeepModel2} can be used for feature engineering while simultaneously modeling correlation structures  using latent Gaussian variables. To this end, we consider genomic and epigenomic data from \textit{Arabidopsis thaliana}.
Arabidopsis is an extremely well-studied model organism for which many genomic and epigenomic data sets are publicly available, see for example~\citeA{becker2011spontaneous}.  DNA locations with a nucleotide of type cytosine nucleobase (C) can be either methylated or not. Our focus will be  on  modeling the number of methylated reads through different covariates including (local) genomic structures, gene classes, and expression levels. The studied data was obtained from the NCBI GEO archive \cite{barrett2013ncbi}, where we consider a sample of $n = 500$ base-pairs chosen from a random genetic region of a single plant. Only cytosine nucleobases can be methylated, hence these 500 observations correspond to 500 sequential cytosine nucleobases from the selected genetic region. 

At each location $t_i$, there are $R_i$ reads of which $Y_i$ are methylated. Although the data might be modeled by  a binomial distribution, we prefer to apply a Poisson distribution for $Y_i$ with mean $\mu_i \in \mathbb{R}^{+}$ with the possibility of including an offset. This demonstrates the ability of the BGNLM approach to work with different probability distributions from the exponential family.
In the extended BGNLMM model \eqref{DeepModel2}, we use the logarithm as the canonical link function. We consider $p=14$ input variables defined as follows: A factor with three levels is coded with two dummy variables $X_1$ and $X_2$. This describes whether a location belongs to a CGH, CHH, or CHG genetic region, where H is either A, C, or T. A second factor is concerned with the distance of the location to the previous cytosine nucleobase (C), where the dummy variables  $X_3-X_8$ are used to code whether the distance  is  2, 3, 4, 5, from 6 to 20, or greater than 20, respectively,  taking a distance of 1 as reference. 
A third factor describes whether a location belongs to a gene, and if yes, whether this gene belongs to a particular group of biological interest. These groups are denoted $M_\alpha$, $M_\gamma$, $M_\delta$ and $M_0$. They are coded by 3 additional dummy variables, $X_{9}-X_{11}$, with $M_0$, the group where there are no genes,  used as a reference.  Two further covariates are derived from the expression level for a nucleobase.  The cutoffs, which define binary covariates $X_{12}$ and $X_{13}$, are either greater than 3000 or greater than 10000 fragments per kilobase of transcripts per million mapped reads. 
The last covariate, $X_{14}$, is an offset defined by the total number of reads per location $R_i \in \mathbb{N}$. The offset is modeled as an additional component of the model. Hence, it can be regarded as a matter of model choice.

We consider the following latent Gaussian variables to model spatial correlations:
\begin{description}
\item[Autoregressive process of order 1:] Assume $ \delta_i =  \rho\delta_{i-1} + \epsilon_i\in \mathbb{R}$ with $\epsilon_{i} \sim N(0,\tau^{-1})$, $i = 1,...,n$ and $|\rho|<1$. For this process, the priors on the hyper-parameters are defined as follows: first, reparametrize to $\psi_1 = \tau(1-\rho^2)$, $\psi_2 =   \log{\tfrac{1+\rho}{1-\rho}}$, then assume
$\psi_1 \sim \text{Gamma}(1,5\times 10^{-5})$, $\psi_2 \sim N(0,0.15^{-1})$.

\item[Random walk of order 1:] Assume independent increments: $\Delta \delta_i = \delta_i - \delta_{i-1}\sim N(0,\tau^{-1})$ with a prior  $\tau \sim \text{Gamma}(1,5\times10^{-5})$. 

\item[Zero-mean Ornstein-Uhlenbeck process:] This is defined via the stochastic differential equation $d\delta(t) = -\phi\delta(t)dt + \sigma dW(t)$, where $\phi>0$ and $\{W(t)\}$ is the Wiener process.  This is the continuous time analogue to the discrete time $AR(1)$ model and the process is Markovian. Let $\delta_1,...,\delta_n$ be the values of the process at increasing locations
 $t_1,...,t_n$ and define $\rho=\exp(-\phi)$ and $\tau = 2\phi/\sigma^2$. Then the conditional distribution $\delta_i|\delta_1,...,\delta_{i-1}$ is Gaussian with mean $\rho^{z_i}\delta_{i-1}$ and precision $\tau(1-\rho^{2z_i})^{-1}$ , where $z_i = t_i - t_{i-1}$. Priors on the hyper-parameters are 
 $\tau \sim \text{Gamma}(1,5\times 10^{-5})$, $\log(\phi) \sim N(0,0.2^{-1})$. 

\item[Independent Gaussian process:] Assume $\delta_i\stackrel{ind}{\sim} N(0,\tau^{-1})$ and prior $\tau \sim \text{Gamma}(1,5\times10^{-5})$.
\end{description} 
These different processes allow different spatial dependence structures of methylation rates along the genome to be modeled. They can also account for the variance which is not explained by the covariates. BGNLMM can be used to find the best combination of latent variables for modeling this dependence in combination with nonlinear feature engineering. 
The Bayesian model is completed with Gaussian priors for the regression coefficients as
\begin{align}
\beta_j|\bm{\gamma} \stackrel{ind}{\sim} & I(\gamma_j=1)N(0,\tau_{\beta}^{-1}),\quad j=1,...,p;\\
\tau_{\beta}\sim & \text{Gamma}(1,5\times10^{-5}).
\end{align}
We then use prior~\eqref{eq:modelprior} with $a=e^{- \log n}$ for $\bm{\gamma}$ and a similar prior for $\bm \lambda$ associated with selection of the latent Gaussian variables with $b = e^{-\log n}$ and each of the $r=4$ latent Gaussian processes having equal measures of complexity. Furthermore, we used $\mathcal{G} = \{\text{sigmoid}(x),\text{gauss}(x),\text{tanh}(x)$,
 $\text{atan}(x),\text{sin}(x),\text{cos}(x)\}$  with uniform $P_\mathcal{G}$, $D = 5$, $ Q=15$, and  $L = 15$ in this example.
 The marginal likelihoods are computed using the INLA approach~\cite{rue2009eINLA}.
\begin{table}[t]
\centering
\begin{tabular}{l|lc}%
\hline 
&Variable&Posterior\\
\hline 
Features&offset(log(total.bases))&1.000\\
&CGH&0.999\\
&CHG&0.952\\
\hline
Latent Gaussian variables&Random walk, order 1&1.000\\
\hline
\end{tabular}
\caption{\label{tepi} Results for Example \ref{ap:ex.epi}: Features and latent Gaussian variables with a posterior probability above 0.25 found by BGNLM using 16 parallel threads.}
\end{table}

There are three features with large posterior probability (Table~\ref{tepi}): the offset for the total number of observations per location as well as two features indicating whether the location is CGH or CHG. Among the latent Gaussian variables, only the random walk process of order one was found to be important. None of the nonlinear features was important for this example. As in Example~\ref{Ex:Spam}, we observe that although our feature space includes highly nonlinear features, the regularization induced by the priors guarantees the choice of parsimonious models. Nonlinear features are only selected if they are necessary. This results in interpretable models. 

\section{{Summary and Discussion}}\label{section5}
In this article, we have introduced a new class of Bayesian generalized nonlinear regression models to perform automated feature generation, model selection, and model averaging in a Bayesian context. The genetically modified mode jumping MCMC algorithm~\cite{hubin2018novel} is adopted to estimate model posterior probabilities. The algorithm combines two key ideas: Having a population (or search space)  of highly predictive features which is regularly updated and using mode jumping MCMC  to efficiently explore models within these populations. 

In several examples, we have shown that the suggested approach can be efficient not only for prediction but also for model inference. 
Inference for BGNLMs often requires significant computational resources, hence parallelization is recommended. The resulting benefits are illustrated in several examples. 
 The penalization on complexity implied by our model prior (\ref{eq:modelprior}) yields so strong regularization that features with a depth larger than three are rarely generated in the examples considered. On the other hand, we have seen that for a maximum depth of $3$ the feature space already becomes huge and will be sufficient to model almost any nonlinear relationship. Furthermore,  features with depth $d > 3$ are not easy to interpret.

One of the main advantages of Bayesian deep learning is the possibility to quantify the uncertainty of predictions. Currently, commonly used  Bayesian approaches to deep learning rely on variational Bayes approximations~\cite{Gal2016Uncertainty}, which tend to be rather crude. 
In contrast, our approach provides well-defined and mathematically justified uncertainty measures for any parameter of interest via standard Bayesian model averaging. This also allows for the calculation of reliable credible intervals, at least for the fully Bayesian approach.
A memory-efficient way of performing parallelized BGNLM is implemented in the R-package \texttt{EMJMCMC}, which is currently available from the GitHub repository \cite{gmjref}. The package gives the user flexibility both in the choice of methods to obtain marginal likelihoods and in the prior specification.

There are still several important questions  open for further research. The first topic is concerned with the choice of model priors in Section \ref{sec:priors}. The specific structure of the prior for features in equation (\ref{glmgammaprior}) includes a parameter $a$ and a complexity measure $c(\cdot)$. For the parameter $a$, we are using $a=e^{-2}$ when interested in prediction and $a=e^{-\log n}$ when interested in model identification, since the latter is more conservative.  These choices are informed by considerations similar to those leading to modifications of AIC and BIC which are controlling FWER  \cite{bogdan2020identifying}. They worked well in the examples we have presented here but better choices could be possible and one might also like to consider introducing reasonable hyper-priors for $a$. The choice of our complexity measure was motivated by the geometric prior used by \citeA{Fritsch1} in the context of logic regression. The number of leaves of a logic tree directly translates into the operations count of a feature in our context. It will be interesting in the future to consider more general complexity measures involving both the width and the depth of features but this appears to be a research project on its own. We only want to mention here that the definition of depth for the multiplication parameter is not as straightforward as for modifications and projections and that there is quite some scope for potential improvement of the model priors. 

The second topic for further research is related to the choice of $\boldsymbol\alpha$ in the feature generating process. We used the pragmatic strategy of fixing parameters in nested features, estimating parameters on the outer layer of the new feature, and thereafter taking a nonlinear modification of the obtained feature. This approach,  inspired by the ideas from \citeA{fahlman1990cascade}, is  computationally efficient and guarantees unique estimates. We have implemented three further strategies, including optimization of weights from the last nonlinear projection, optimization with respect to all layers of a feature, and a fully Bayesian approach where all of the weights across all layers of the features are considered as model parameters. The second and third strategies are computationally more demanding than the simplest strategy and require additional assumptions on the nonlinear transformations involved. The fourth strategy provides a fully Bayesian approach which is theoretically sound but extremely slow in terms of convergence. However, by excluding projections one can reduce the model space and obtain a fully Bayesian framework, which might be useful for some applications. Studying the properties of this restricted class of BGNLM models with no projections involved is thus of interest for further research.

In Appendix~\ref{ap:add.ex1_3}, we compare the performance of BGNLM when using more complex strategies for estimating $\bm\alpha$ parameters. Interestingly, none of these strategies clearly outperforms the simple baseline strategy used in the main part of our applications. To some extent, this has to do with the complexity measure that we have used on our features within the model prior~\eqref{eq:modelprior}, which results in rather high penalties for projections compared with modifications and multiplications. Consequently, the majority of nonlinear features we have obtained in our examples do not involve projection transformations. Hence the estimation strategy of $\bm\alpha$ parameters is of minor importance. In the future, we plan to work with complexity measures for which it becomes less costly to add $\bm\alpha$ components in the projection. This will go along with further research including simulation scenarios where nested projections are part of the data generating model. 

An important difference between our approach and deep learning is that BGNLM does not fix the structure of multi-layer neural networks in advance but has the potential of learning the network structure when generating new nonlinear features.
By excluding new features that are linear combinations of previously defined features, BGNLM only includes features with \emph{different} topologies, while standard neural networks only include features with \emph{similar} topologies (but different weight parameters). In the fully Bayesian version of BGNLM,  features with similar topologies can be included, giving a Bayesian generalization of neural networks with a possibility of learning the network structure as well. Utilizing this option will, however, require more efficient algorithms.

Another important issue left for discussion is how to manage very large data sets with the BGNLM approach.
As for the marginal likelihood calculated with respect to parameters across all of the layers, only very crude approximate solutions based on the variational Bayes approach \cite{jordan1999introduction} are currently scalable for such problems~\cite{barber1998ensemble,blundell2015weight}. \citeA{mackay1992practical,denker1991transforming} applied the Laplace approximations to approximate marginal likelihood across all layers. This approach is also computationally very demanding and cannot easily be combined with the combinatorial search for the best models. \citeA{neal2012bayesian} suggested Hamiltonian Monte Carlo (HMC) to make proper Bayesian inference on Bayesian neural networks. Unfortunately, his approach is even more computationally demanding and hence does not seem scalable to high-dimensional model selection. To reduce the computational complexity of HMC and improve its scalability to large data sets, \citeA{welling2011bayesian} suggested using stochastic estimates of the gradient of the likelihood. 
Many recent articles describe the possibility of such sub-sampling combined with MCMC  \cite{quiroz2019speeding,quiroz2017speeding,quiroz2016exact,flegal2012applicability,pillai2014ergodicity}, where unbiased likelihood estimates are obtained from subsamples of the whole data set in such a way that ergodicity and the desired limiting properties of the MCMC algorithm are maintained. These methods are not part of the current implementation of BGNLM, but our approach can be adapted with a relatively moderate effort to allow sub-sampling MCMC techniques.


\appendix

\section{Computational Complexity}\label{ap:hyper}

\begin{table}[H]
\caption{Comparison of computational complexities for model building (training phase) and prediction (test phase) for the different algorithms.}
\label{tab:comptime}
\resizebox{\textwidth}{!}{%
\begin{tabular}{lllll}
\multicolumn{1}{l}{\textbf{Method}} &
\multicolumn{1}{l}{\textbf{Alias}} &
  \multicolumn{1}{l}{\textbf{Classification/Regression}} &
  \multicolumn{1}{l}{\textbf{Training}} &
  \multicolumn{1}{l}{\textbf{Prediction}} \\\hline
Random Forest &RFOREST &	C+R   &	$O(n^2p n_{trees})$\footnote{$n_{trees}$ is the number of classification and regression trees involved} & $O(p n_{trees})$\\
Gradient Tree Boosting & TXGBOOST &  C+R & $O(npn_{trees})$ & $O(pn_{trees})$\\
Gradient Linear Boosting & LXGBOOST &  C+R & $O(p^2n+p^3)$ & $O(p)$\\ 
Logistic Regression &LR& 	C 	& $O(p^2n+p^3)$ & $O(p)$\\ 
Gaussian Regression &GR& 	R 	& $O(p^2n+p^3)$ & $O(p)$\\ 
Neural Network &DEEPNETS& 	C+R  &	 $O(w^2n+w^3)$\footnote{$w=pp_{l_1}+\sum_{i \ in 1,...,n_{layers}}p_{l_i}p_{l_{i+1}}$}  &	$O(w)$\\
Naive Bayes & NBAYES & C & $O(np+2p)$ & $O(p)$\\
Variational Bayes Gaussian Regression &VARBAYES& 	R 	& $O(p^2n+p^3)$ & $O(p)$\\ 
Lasso Regression & LASSO & C+R & $O(p^2n+p^3)$ & $O(p)$\\ 
Ridge Regression & RIDGE & C+R & $O(p^2n+p^3)$ & $O(p)$\\ 
Bayesian Generalized Linear Model & BGLM & C+R &$O(2^p(p^2n+p^3))$&$O(p2^p)\text{ or }O(p)$\footnote{full model averaging is considered before "or", whilst the median probability model or model averaging of the small number of "best" models is considered after "or".}\\
Bayesian Generalized Nonlinear Model & BGNLM & C+R &$O(2^q(Q^2n+Q^3))$\footnote{in practice, one would only run inference for a fixed number of iterations and per iteration, the complexity is $O(Q^2n+Q^3)$. }&$O(Q2^q)\text{ or }O(Q)$\footnote{full model averaging is considered before "or", whilst the median probability model or model averaging of the small number of "best" models is considered after "or".}\\
\hline
\end{tabular}
}
\end{table}

\section{Further Applications}\label{ap:furtherexamp}

In Appendix~\ref{Ex:NeoAsteroids}, we provide an additional example of classification of Asteroids. Moreover, we address two extra examples of model inference: In the first of them (Appendix~\ref{Ex:JupiterMass}), we recover the planetary mass law based on the exoplanets data introduced in Section~\ref{Ex:Kepler}. In the latter (Appendix~\ref{ap:sim.ex}), we address a complex simulation scenario of a logic regression case \cite{hubin2018novel}, where we show the ability of BGNLM to recover highly nonlinear Boolean interactions.

\subsection{Example 6: Neo Asteroids Classification} \label{Ex:NeoAsteroids}
The data set \cite{Neodata} addressed in this example consists of characteristic measures of 20\,766 asteroids, some of which are classified as potentially hazardous objects, whilst others are not. Measurements of the following nine explanatory variables are available: \textit{Mean anomaly, Inclination, Argument of perihelion, Longitude of the ascending node, Rms residual, Semi-major axis, Eccentricity, Mean motion, Absolute magnitude}. It can be downloaded from \url{2016.spaceappschallenge.org}.

\begin{table}[h!]{
\resizebox{\textwidth}{!}{%
\begin{tabular}{llll}%
\hline 
Algorithm&ACC&FNR&FPR\\\hline
BGLM&0.9999 (0.9999,0.9999)&0.0001 (0.0001,0.0001)&0.0002 (0.0002,0.0002)\\
BGNLM{\_}PRL&0.9998 (0.9986,1.0000)&0.0002 (0.0001,0.0021)&0.0000 (0.0000,0.0000)\\
BGNLM&0.9998 (0.9942,1.0000)&0.0002 (0.0001,0.0082)&0.0002 (0.0000,0.0072)\\
LASSO&0.9991 (-,-)&0.0013 (-,-)&{0.0000} (-,-)\\
RIDGE&0.9982 (-,-)&0.0026 (-,-)&0.0000 (-,-)\\
LXGBOOST&0.9980 (0.9980,0.9980)&0.0029 (0.0029,0.0029)&0.0000 (0.0000,0.0000)\\
LR&0.9963 (-,-)&0.0054 (-,-)&0.0000 (-,-)\\
DEEPNETS&0.9728 (0.8979,0.9979)&0.0384 (0.0018,0.1305)&0.0000 (0.0000,0.0153)\\
TXGBOOST&0.8283 (0.8283,0.8283)&0.0005 (0.0005,0.0005)&0.3488 (0.3488,0.3488)\\
RFOREST&0.8150 (0.6761,0.9991)&0.1972 (0.0003,0.3225)&0.0162 (0.0000,0.3557)\\
NBAYES&0.6471 (-,-)&0.0471 (-,-)&0.4996 (-,-)\\
\hline
\end{tabular}}
}
\caption{\label{t1}Comparison of performance (ACC, FPR, FNR) of different algorithms for NEO objects data (Example \ref{Ex:NeoAsteroids}). See caption of Table~\ref{t2} for details.}
\end{table}
 The training sample consisted of $n=64$ objects (32 of which are potentially hazardous objects, whilst the other 32 are not) and the test sample of the remaining $n_p=20\,702$ objects. Table~\ref{t1} shows that even with such a small training set most methods tend to perform very well. The na{\i}ve Bayes classifier has the lowest accuracy with a huge number of false positives. The tree-based methods also have comparably small accuracy, where tree-based gradient boosting, in addition, delivers too many false positives. Random forests tend to have on average too many false negatives, though there is a huge variation of performance between different runs ranging from almost perfect accuracy down to accuracy as low as the na{\i}ve Bayes classifier. In this example, we used $\mathcal{G} = \{\text{gauss}(x),\text{tanh}(x)$,
$\text{atan}(x),\text{sin}(x)\}$, $a = e^{-2}$, $D=4$, $L=15$ and $Q = 15$.
 
 The BGNLM model is among the best methods for this data set, while BGLM has the best median performance. This indicates that nonlinear structures do not play an important role in this example and all the other algorithms based on linear features (LASSO, RIDGE, logistic regression, linear-gradient boosting)  performed similarly well. BGLM gives the same result in all simulation runs, the parallel version of GMJMCMC for BGNLM gives almost the same model as BGLM and only rarely adds some nonlinear features, whereas the single-threaded version of GMJMCMC for BGNLM much more often includes nonlinear features (Table \ref{Tab:complexityAst}). The slight variation between simulation runs for the single-threaded version of GMJMCMC suggests that despite the generally good performance of BGNLM the algorithm has not fully converged in some runs.

\begin{table}[!ht]
\centering
\begin{tabular}{cc}
\begin{tabular}{lrrr}%
\multicolumn{4}{l}{\underline{\textbf{Example 6}: Asteroid}}\\
compl.&BGNLM&BGNLM{\_}PRL&BGLM\\
1&8.96&9.00&9.00\\
2&2.58&0.05&0.00\\
\hline
Total&11.54&9.05&9.00
\\
\hline
\end{tabular}
\end{tabular}
\caption{\label{Tab:complexityAst}Mean frequency distribution of feature complexities detected by the different BGNLM algorithms in 100 simulation runs for Asteroid data (Example \ref{Ex:NeoAsteroids}). The final row for each example gives the mean of the total number of features in 100 simulation runs which had a posterior probability larger than 0.1.}
\end{table}

As shown in Table~\ref{Tab:complexityAst}, all reported nonlinear features had a complexity of 2. As mentioned previously the parallel version of BGNLM detected way fewer nonlinear features than the simple versions. This suggests that GMJMCMC has not completely converged in some simulation runs. Approximately half of the nonlinear features were modifications and the other half were multiplications. In this example, not a single detected projection was found significant by the GMJMCMC. 
 
\subsection{Example 7: Jupiter Mass of the Planet}  \label{Ex:JupiterMass}

In this example, we consider the planetary mass as a function of its radius and density. It is common in astronomy to use the measures of Jupiter as units and a basic physical law gives the nonlinear relation 
\begin{align} \label{mass_law}
	m_p\approx R^3_p\times \rho_p \;.
\end{align}
Here, just as described in Table~\ref{tab:alias}, $m_p$ is the planetary mass measured in units of Jupiter mass (denoted \textit{PlanetaryMassJpt} from now on). Similarly, the radius of the planet $R_p$  is measured in units of Jupiter radius  and the density of the planet $\rho_p$  is measured in units of Jupiter density. Hence in the data set the variable  \textit{RadiusJpt} refers to $R_p$, and \textit{PlanetaryDensJpt} denotes $\rho_p$. The approximation sign is used because the planets are not exactly spherical but rather almost spherical ellipsoids.

A BGNLM with a Gaussian observation model and identity link function is used to model \textit{PlanetaryMassJpt} as a function of the following ten potential input variables: \textit{TypeFlag}, \textit{RadiusJpt}, \textit{PeriodDays}, \textit{SemiMajorAxisAU}, \textit{Eccentricity}, \textit{HostStarMassSlrMass}, \textit{HostStarRadiusSlrRad}, \textit{HostStarMetallicity}, \textit{HostStarTempK}, \textit{PlanetaryDensJpt} (see Table~\ref{tab:alias} for details). 
To illustrate to which extent the performance of BGNLM depends on the number of parallel runs, we furthermore consider computations with 1, 4, and 16 threads, respectively. To evaluate the capability of BGNLM to detect true signals, we run the algorithm for a given number of threads for $N = 100$ times. 

In each of the threads, the algorithms were first run for 10\,000 iterations, with population changes at every 250 iterations, and then for a larger number of iterations based on the last population (until a total number of 10\,000 unique models was obtained). Results for BGNLM using different numbers of threads are summarized in Table \ref{Tab:Results_Ex2} for $\mathcal{G}_1$ and $\mathcal{G}_2$, which are the same as in Example~\ref{Ex:Kepler}. All other tuning and hyper-parameters of the model and the algorithm are also the same as in Example~\ref{Ex:Kepler}.

\begin{table}[h!]
\centering
\begin{tabular}{lccc|ccc}
\hline
&\multicolumn{2}{c}{BGNLM{\_}PRL}&$\mathcal{G}_1$&\multicolumn{2}{c}{BGNLM{\_}PRL}&$\mathcal{G}_2$\\\hline
Threads&Pow&FP&FDR&Pow&FP&FDR\\\hline
16&1.00 &0.00&0.00&0.93&0.36&0.215\\%
4&0.79&0.40&0.21 &0.69&0.49&0.34\\%
1&0.42&1.21&0.58 &0.42&1.25&0.58\\\hline
\end{tabular}
\caption{\label{Tab:Results_Ex2}Pow, FP, and FDR for detecting the mass law in Equation \eqref{mass_law} based on the decision rule that the posterior probability of a feature  is larger than $\eta^* = 0.25$ (Example 7). The feature $R\times R\times R\times \rho_p$  is counted as true positive, all other selected features as false positive. BGNLM is applied using the nonlinear sets (NL set) $\mathcal{G}_1$ and $\mathcal{G}_2$ and different numbers of parallel threads.} 
\end{table}

Clearly, the more resources are available the better BGNLM performs. 
BGNLM manages to find the correct model with rather large power (reaching gradually one) and small FDR (reaching gradually zero) when the number of parallel threads is increased. When using only a single thread, it often happens that instead of the correct feature some closely related features are selected (see the Excel sheet \texttt{Mass.xlsx} in the supplementary material \cite{supt} for more details). Results for the set $\mathcal{G}_1$ are slightly better than for $\mathcal{G}_2$ which illustrates the importance of having a good set of transformations when interested in model inference. The power is lower and FDR is larger for $\mathcal{G}_2$ which is mainly due to the presence of $|x|^{7/2}$ in the set of nonlinearities. The feature $R^{7/2}_p\times \rho_p$ is quite similar to the correct law \eqref{mass_law}  and moreover has lower complexity than the feature $R^3 \rho_p$ due to how $\mathcal{G}_2$ is defined. Hence, it is not surprising that it is often selected, specifically when BGNLM was not run sufficiently long to fully explore features with larger complexity.

\subsection{Example 8: Simulated Data With Complex Combinatorial Structures}\label{ap:sim.ex}

\begin{table}[h!]
\centering
\begin{tabular}{lcc}
\hline
&BGNLM&BLRM\\\hline
$X_7$&1.0000&0.9900\\
$X_8$&1.0000&1.0000\\
$X_2X_9$&1.0000&1.0000\\
$X_{18}X_{21}$&1.0000&0.9600\\
$X_{1}X_{3}X_{27}$&1.0000&1.0000\\
$X_{12}X_{20}X_{37}$&1.0000&0.9900\\
$X_{4}X_{10}X_{17}X_{30}$&0.9900&0.9100\\
$X_{11}X_{13}X_{19}X_{50}$&0.9800&0.3800\\
Overall power&0.9963&0.9038\\
FP&0.5100&1.0900\\ 
FDR&0.0601&0.1310\\
\hline\\
\end{tabular}
\caption{\label{simres}Results for Example~\ref{ap:sim.ex}. Pow for individual trees, overall power (average power over trees),  FP, and FDR are compared between BGNLM and Bayesian logic regression.} 
\end{table}
In this simulation study, we generated $N = 100$ data sets  with $n=1\,000$ observations and $p=50$ binary covariates. The covariates were assumed to be independent and were simulated for each simulation run  as $X_{j}\sim \text{Bernoulli}(0.5)$ for $j  = 1,\dots,50$. In the first simulation study the responses were simulated according to a Gaussian distribution with error variance $\sigma^2 = 1$ and individual expectations specified as follows: 
\begin{eqnarray*}
E\left\{Y|\boldsymbol{X}\right\} & = & 1 + 1.5 X_{7} + 1.5 X_{8}+ 6.6 X_{18}X_{21} + 3.5 X_{2}X_{9}
+ 9 X_{12}X_{20}X_{37} +\\ 
&& 7 X_{1}X_{3}X_{27} 
+ 7 X_{4}X_{10}X_{17}X_{30} 
+ 7 X_{11}X_{13}X_{19}X_{50}.
\end{eqnarray*}
For BGNLM, we set $\mathcal{G} = \{\text{sigmoid}(x)\,\text{gauss}(x),\text{tanh}(x)$,
 $\text{atan}(x),\text{sin}(x),\text{cos}(x)\}$, with a uniform $P_\mathcal{G}$. We also use $a = e^{-\log n}$, $D = 4$, $L = 40$, and $Q = 40$.
We compare the results of  BGNLM with those for the Bayesian logic regression model in \citeA{hubin2018novel}. The latter model differs from the current one in that the model prior is different. For a given logical tree (which is the only allowed feature form), we use $a^{c(L_j)} = (N(s_j))^{-1}  \; , \quad  s_j\leq C_{max}$, where $N(s)=\binom{m}{s}\ 2^{2s-2}$. $Q$ and priors for the model parameters are the same as defined in the BGNLM model. All algorithms were run on 32 threads until the same number of models were visited after the last change of the model space. In particular, in each of the threads, the algorithms were run until $20\,000$ unique models were obtained after the last population of models had been generated at iteration $15\,000$. Specification of the Bayesian Logic Regression model corresponds exactly to the one used in simulation Scenario 6 in \citeA{hubin2018novel}. In this example, a detected feature is only counted as a true positive if it exactly coincides with a feature of the data generating model. The results are summarized in Table~\ref{simres}.  Detection in this example corresponds to the features having marginal inclusion probabilities above $0.5$ after the search is completed.

GMJMCMC performed exceptionally well for fitting this BGNLM. 
The version of GMJMCMC algorithm for fitting Bayesian Logic Regression (BLRM) from \citeA{hubin2018novel} in this case performed almost as well as GMJMCMC for BGNLM, except for a significant drop in power in one of the four-way multiplications.  This is however not too surprising because the multiplication transformation of BGNLM models perfectly fits the data generating model whereas the logic regression model focuses on general logic expressions and provides in that sense a larger chance to generate features which are closely related to the data generating four-way multiplication \cite{hubin2018novel}.

\section{Results for Alternative Strategies of Specifying Weights}\label{ap:add.ex1_3}

In this section, the predictive performance of the four alternative strategies for specification of $\bm\alpha$ is compared (with the first three given in Section~\ref{sub:Estimate_alpha} and the last one in Section~\ref{sec:fullb}).
Tables~\ref{t12}-\ref{t11} show the result for the breast cancer data (Example~\ref{Ex:BreastCancer}), the spam data (Example~\ref{Ex:Spam}), and the NEO asteroids classification problem (Example~6), respectively. Comparing  Table~\ref{t12} with Table~\ref{t2}, Table~\ref{t13} with Table~\ref{t3}, and Table~\ref{t11} with Table~\ref{t1}, we see that there is no substantial difference in predictive performance between the strategies used for specifying weights for the three addressed data sets. On one hand, this might indicate that there is no real difference in which of the strategies we use for the optimization of projection-based features. But, on the other hand, this might indicate that these features, in general, do not play a big role in predictions as a result of very strong regularization of them. Further research on this is required.

\begin{table}[h!]{
\caption{\label{t12}Comparison of performance  (ACC, FPR, FNR) of alternative feature engineering strategies  (indicated with $\_2,\_3,\_4$ in the table)  for Example~\ref{Ex:BreastCancer}. For methods with the random outcome, the median measures (with minimum and maximum in parentheses) are displayed.
The algorithms are sorted according to median power.}
\resizebox{\textwidth}{!}{%
\begin{tabular}{llll}%
\hline 
Algorithm&ACC&FNR&FPR\\\hline
BGNLM{\_}1&0.9695 (0.9554,0.9789)&0.0536 (0.0479,0.0809)&0.0148 (0.0037,0.0326)\\
BGNLM{\_}3&0.9695 (0.9507,0.9789)&0.0536 (0.0479,0.0862)&0.0148 (0.0000,0.0361)\\
BGNLM{\_}4&0.9671 (0.9577,0.9789)&0.0536 (0.0305,0.0756)&0.0184 (0.0000,0.0361)\\
BGNLM{\_}2&0.9671 (0.9531,0.9789)&0.0536 (0.0422,0.0862)&0.0184 (0.0000,0.0361)\\
\hline
\end{tabular}
}}
\end{table}

\begin{table}[h!]{
\caption{\label{t13}Comparison of performance (ACC, FPR, FNR) of alternative feature engineering strategies   for Example~\ref{Ex:Spam}.  See caption of Table~\ref{t12} for details.}
\resizebox{\textwidth}{!}{
\begin{tabular}{llll}%
\hline 
Algorithm&ACC&FNR&FPR\\\hline
BGNLM{\_}1&0.9243 (0.9113,0.9328)&0.0927 (0.0808,0.1116)&0.0552 (0.0465,0.0658)\\
BGNLM{\_}2&0.9243 (0.9100,0.9357)&0.0927 (0.0780,0.1103)&0.0545 (0.0445,0.0686)\\
BGNLM{\_}3&0.9237 (0.9100,0.9321)&0.0924 (0.0766,0.1122)&0.0548 (0.0474,0.0714)\\
BGNLM{\_}4&0.9237 (0.9113,0.9315)&0.0931 (0.0821,0.1077)&0.0562 (0.0470,0.0714)\\
\hline
\end{tabular}
}}
\end{table}

\begin{table}[h!]{
\centering
\resizebox{\textwidth}{!}{
\begin{tabular}{llll}%
\hline 
Algorithm&ACC&FNR&FPR\\\hline
BGNLM{\_}3&0.9998 (0.9959,1.0000)&0.0002 (0.0001,0.0056)&0.0002 (0.0000,0.0042)\\
BGNLM{\_}1&0.9998 (0.9942,1.0000)&0.0002 (0.0001,0.0082)&0.0002 (0.0000,0.0072)\\
BGNLM{\_}2&0.9998 (0.9933,1.0000)&0.0002 (0.0001,0.0089)&0.0002 (0.0000,0.0048)\\
BGNLM{\_}4&0.9998 (0.9932,0.9999)&0.0002 (0.0001,0.0097)&0.0002 (0.0000,0.0042)\\
\hline
\end{tabular}}}
\caption{\label{t11} Comparison of performance (ACC, FPR, FNR) of alternative feature engineering strategies for Example~6. See caption of Table~\ref{t12} for details.}
\end{table}

\bigskip
\section*{Supplementary Material}

\noindent \textbf{R package:} \textit{R} package \textit{EMJMCMC} for doing inference in the BGNLM model (R)(G) MJMCMC \cite{gmjref} which is available on GitHub at \url{http://aliaksah.github.io/EMJMCMC2016}.

\noindent \textbf{Data and code:} supplementary data and code for all the examples as well as the excel sheets for all of the results are given in \citeA{supt} which is available on GitHub at \url{https://github.com/aliaksah/EMJMCMC2016/tree/master/supplementaries/BGNLM}.

\section*{Acknowledgements} We thank the CELS project (\url{http://www.mn.uio.no/math/english/research/groups/cels/}) at the University of Oslo for giving us the opportunity, inspiration, and motivation to write this article. We would also like to acknowledge NORBIS (\url{https://norbis.w.uib.no/}) for funding an academic stay of the first author in Vienna. We are grateful to Prof. Robert Feldt for providing us links to symbolic regression. Finally, we thank Dr. Anders L{\o}land (Norwegian Computing Center), Dr. Kory D. Johnson (Vienna University of Economics and Business), and Dr. Johan Pensar (University of Oslo) for proofreading the article and checking its language. Last but not least, we also acknowledge the HPC cluster from \url{sigma2.no} for providing us with the computational resources used for obtaining the results of this paper. More specifically, we thank projects NN9862K and NN9244K.

\vskip 0.2in
\bibliography{sample}
\bibliographystyle{theapa}

\end{document}